\documentclass[lettersize,journal,comsoc]{IEEEtran}
\usepackage{amsmath,amsfonts}

\DeclareMathOperator*{\argmin}{arg\,min}

\usepackage{algorithmic}
\usepackage{algorithm}
\usepackage{array}
\usepackage[caption=false,font=normalsize,labelfont=sf,textfont=sf]{subfig}
\usepackage{textcomp}
\usepackage{stfloats}
\usepackage{url}
\usepackage{verbatim}
\usepackage{graphicx}
\usepackage{cite}
\usepackage{booktabs}
\usepackage{placeins}

\newtheorem{theorem}{Theorem}
\newtheorem{proposition}[theorem]{Proposition}
\newtheorem{corollary}[theorem]{Corollary}

\newtheorem{lemma}[theorem]{Lemma}
\begin{document}

\title{Model-Adaptive and Risk-Constrained Frequency Hopping Against Predictive Jammers}

\author{Chen~Yanbo and Zhou~Xinjing%
	\thanks{Chen Yanbo is with the School of Electrical and Electronic Engineering,
		Nanyang Technological University, Singapore
		(e-mail: YCHEN123@e.ntu.edu.sg).}%
	\thanks{Zhou Xinjing is with Universit\`a di Udine, Udine, Italy
		(e-mail: zhou.xinjing@spes.uniud.it).}}

\maketitle

\begin{abstract}
	Adaptive frequency hopping against predictive jamming must address both
	model uncertainty and policy exposure: the context--loss relationship
	may vary across operating regimes, while persistent hopping patterns may
	expose high-probability channels to attack. We propose D-PACT-AFH, a
	model-adaptive and risk-constrained adversarial contextual-bandit
	framework in which a Tsallis-FTRL master combines a global linear learner
	with a partitioned local learner and selects the model class online.
	D-PACT-Hit incorporates channel-wise marginal hit risk into model
	selection, while D-PACT-Safe applies a minimum-Kullback--Leibler
	projection to enforce a per-slot risk budget. We establish estimator
	validity under non-anticipating attacks, an oracle decomposition relative
	to the better fixed base learner on the common execution trajectory, and
	an exact conditional-risk guarantee for the
	Safe projection. Experiments across diverse channel regimes and jammer
	types demonstrate effective model adaptation and a controllable
	goodput--risk tradeoff: D-PACT-AFH recovers \(95.5\%\) of the local
	learner's gain under observable switching while avoiding \(77.7\%\) of
	its degradation in a negative-control regime.
\end{abstract}



\section{Introduction}
\label{sec:introduction}

\IEEEPARstart{A}{daptive} frequency hopping is a standard defense
against interference and intentional jamming. Its effectiveness depends
on the transmitter's ability to identify favorable channels while
avoiding exploitable hopping patterns. Online learning has consequently
been used to formulate frequency hopping as an adversarial bandit
problem, adapt channel access under unknown environments, jointly select
channels and transmission rates, and process richer spectrum
observations through deep reinforcement learning
\cite{wang2012ufh,zhou2016unknown,hanawal2016joint,
	xu2020intelligent,qi2024hopping}.

Two limitations remain. First, most learning-based schemes fix the
learner's model class in advance. A global linear contextual model shares
observations efficiently across channels and regimes, and methods such as
LC-Tsallis-INF provide strong guarantees when this representation is
adequate \cite{kato2025lc}. The same sharing can become systematically
misspecified when channel responses vary across observable regimes.
Locally specialized models reduce this approximation bias, but fragment
the observations and incur higher finite-horizon estimation cost.
Adaptive frequency hopping must therefore determine online whether
global sharing or local specialization is more appropriate.

Second, low communication loss does not imply low attack exposure.
Learning-based jammers can identify effective physical-layer actions
\cite{amuru2016jamming} or infer recurring channel-hopping behavior
\cite{bout2022folpetti}. A transmitter may therefore achieve high nominal
utility while assigning excessive probability to predictably attacked
channels. Existing safety-aware and constrained contextual-bandit
methods provide general risk-control mechanisms
\cite{sun2017safety,thornton2022constrained}, but do not jointly address
contextual-model uncertainty and predictive-jamming exposure.

We develop D-PACT-AFH, a model-adaptive and risk-constrained
frequency-hopping framework for non-anticipating predictive jammers. A
Tsallis-FTRL master combines a globally shared linear learner with a
partitioned local learner and adapts their weights from bandit feedback.
D-PACT-Hit incorporates channel-wise marginal hit risk into model
selection, whereas D-PACT-Safe projects the complete mixed policy onto a
per-slot risk budget through a minimum-Kullback--Leibler perturbation.

The main contributions are as follows.

\begin{itemize}
	\item We formulate predictive-jamming-resistant frequency hopping as
	an adversarial contextual bandit with an explicit slot-wise
	information pattern. Under the fixed-budget attack model, the jammer
	commits its hidden attack set before the transmitter forms and
	samples its current policy, and the resulting channel-wise marginals
	define the exact conditional hit risk.
	
	\item We develop an online model-selection architecture that adapts
	between global and local contextual learners. Its hit-aware mode
	incorporates predictable exposure into the master objective, while
	its safe mode enforces a policy-level risk budget through a unique
	minimum-KL projection.
	
	\item We establish conditional estimator validity, an oracle-style
	decomposition relative to the better fixed base on the common
	execution trajectory, and an exact conditional hit-risk guarantee
	for feasible Safe projections. Experiments across channel regimes,
	jammer types, risk budgets, and system scales demonstrate adaptation
	under model mismatch and controllable goodput--risk tradeoffs.
\end{itemize}

The remainder of this paper is organized as follows.
Section~\ref{sec:system_model} presents the system and adversary models.
Section~\ref{sec:architecture} develops D-PACT-AFH and its risk-aware
modes. Section~\ref{sec:theory} provides the performance analysis.
Section~\ref{sec:experiments} reports the numerical results, and
Section~\ref{sec:conclusion} concludes the paper.


\section{Related Work}
\label{sec:related_work}

\textit{Learning-based anti-jamming communications:}
Online learning has been widely applied to wireless adaptation without
an accurate model of the channel or jammer. Existing work formulates
uncoordinated frequency hopping as an adversarial multi-armed bandit
\cite{wang2012ufh}, addresses unknown stochastic and adversarial
environments \cite{zhou2016unknown}, jointly adapts hopping frequencies
and transmission rates \cite{hanawal2016joint}, improves bandit-based
hopping efficiency \cite{odeyomi2020mitigating}, and uses deep
reinforcement learning for richer observations and action spaces
\cite{xu2020intelligent,qi2024hopping}. These methods generally retain a
prespecified learner or representation rather than adapting the
contextual model class online.

\textit{Contextual learning and model adaptation:}
Linear contextual bandits improve sample efficiency by sharing
observations through a common parameterization, and LC-Tsallis-INF
provides best-of-both-worlds guarantees under a global linear model
\cite{kato2025lc}. Misspecified contextual bandits allow bounded
deviations from linear realizability \cite{takemura2021misspecified},
whereas model-selection methods adapt among candidate classes
\cite{foster2019model,muthukumar2022model}. Master--base frameworks such
as CORRAL likewise combine bandit algorithms under partial feedback
\cite{agarwal2017corral,pacchiano2020model}. In D-PACT-AFH, however, the
bases encode distinct structural assumptions: the global learner pools
observations across regimes, while the local learner specializes over
partitioned context regions. Both must be updated from the same
off-policy bandit feedback and coordinated with jamming-risk control.

\textit{Predictive jamming and constrained learning:}
Learning has also been studied from the attacker's perspective.
Jamming Bandits learns effective physical-layer jamming actions
\cite{amuru2016jamming}, whereas FOLPETTI predicts recurring
channel-hopping behavior \cite{bout2022folpetti}. Separately,
safety-aware adversarial contextual bandits impose risk constraints
\cite{sun2017safety}, and constrained contextual learning has been
applied to adaptive radio-waveform selection
\cite{thornton2022constrained}. These works treat adaptive attacks,
model uncertainty, and constrained decisions largely in isolation.
D-PACT-AFH jointly adapts the contextual model class, incorporates
channel-wise marginal hit risk into learning, and constrains the final
executed hopping policy.


\section{System Model and Problem Formulation}
\label{sec:system_model}

\subsection{Contextual Frequency-Hopping System}

We consider a transmitter--receiver pair operating over a set
$\mathcal{K}=[K]$ of candidate channels for $T$ time slots.
At the beginning of slot $t$, the learner observes the public context matrix
\[
\mathbf{X}_t
=
[\mathbf{x}_{t,1},\ldots,\mathbf{x}_{t,K}]
\in\mathbb{R}^{d\times K},
\]
where $\mathbf{x}_{t,k}$ describes the observable state of channel $k$,
including relative signal quality, delay quality, and availability.
The complete within-slot information pattern, including public attack
prediction, jammer commitment, policy formation, private action sampling,
and feedback, is specified below and illustrated in
Fig.~\ref{fig:system_model}.

\begin{figure*}[t]
	\centering
	\includegraphics[width=0.98\textwidth]
	{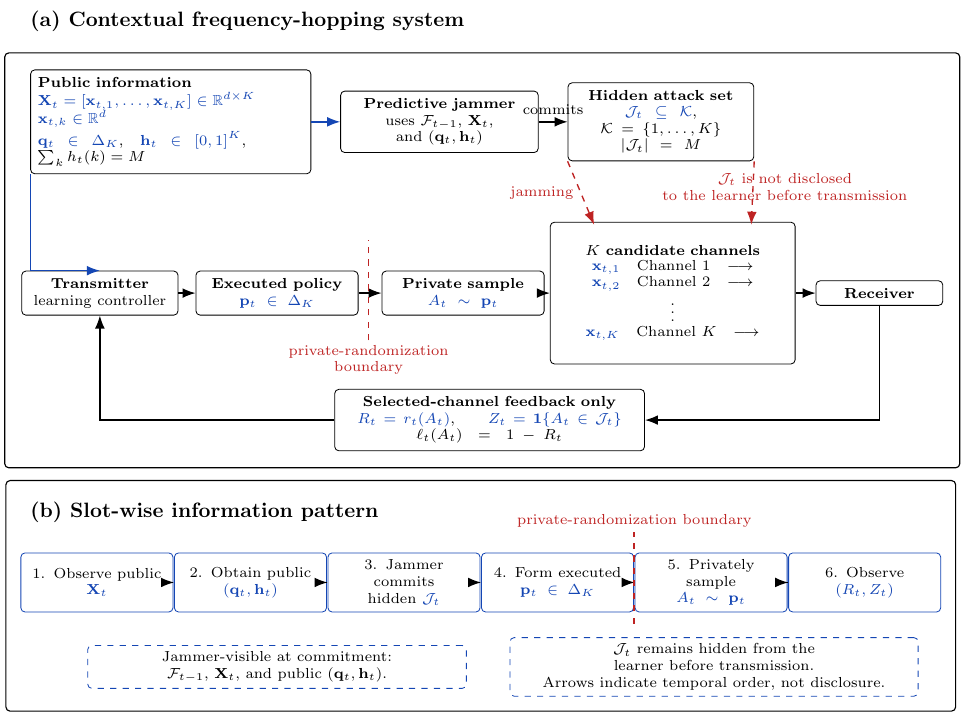}
	\caption{Contextual frequency-hopping system and within-slot information
	pattern. The learner observes the public contexts $\mathbf{X}_t$ and
	obtains the attack prediction $\mathbf{q}_t$ and marginal hit-risk vector
	$\mathbf{h}_t$. The jammer commits the hidden attack set $\mathcal{J}_t$
	before the learner forms the executed policy $\mathbf{p}_t$ and privately
	samples $A_t\sim\mathbf{p}_t$. The learner observes only the selected-channel
	reward $R_t=r_t(A_t)$ and hit indicator
	$Z_t=\mathbf{1}\{A_t\in\mathcal{J}_t\}$; the corresponding observed loss is
	$\ell_t(A_t)=1-R_t$. The attack set remains hidden from the learner before
	transmission. Arrows in panel (b) indicate temporal order, not disclosure
	of the hidden attack set.}
	\label{fig:system_model}
\end{figure*}

\subsection{Public Attack Prediction and Hidden Attack Realization}

Before the learner forms $\mathbf{p}_t$, a public attack predictor
produces
\begin{equation}
	\mathbf{q}_t
	=
	\mathcal{Q}_t(\mathcal{F}_{t-1},\mathbf{X}_t)
	\in\Delta_K,
	\label{eq:predictor_distribution}
\end{equation}
where $\mathcal{F}_{t-1}$ denotes the interaction history available
before slot $t$. The vector $\mathbf{q}_t$ represents the attacker's
predicted channel preference and is distinct from the transmitter's
hopping policy.

For a fixed attack budget $M\in\{0,\ldots,K-1\}$, the predictor is
converted into marginal channel-inclusion probabilities
$\mathbf{h}_t\in[0,1]^K$ according to
\begin{equation}
	h_t(k)=\min\{1,\lambda_t q_t(k)\},
	\qquad
	\sum_{k=1}^{K}h_t(k)=M,
	\label{eq:attack_inclusion_probability}
\end{equation}
where $\lambda_t\geq0$ is chosen to satisfy the attack budget. A
dependent-rounding procedure then draws a hidden set
$\mathcal{J}_t\subseteq\mathcal{K}$ such that
$|\mathcal{J}_t|=M$ and
\begin{equation}
	\Pr\!\left(
	k\in\mathcal{J}_t
	\mid
	\mathcal{F}_{t-1},\mathbf{X}_t,\mathbf{q}_t
	\right)
	=
	h_t(k).
	\label{eq:attack_set_marginal}
\end{equation}
Hence, $\mathbf{h}_t$ is the exact conditional marginal hit-risk vector
under the fixed-budget attack model.

The within-slot order is
\[
\mathbf{X}_t
\rightarrow
(\mathbf{q}_t,\mathbf{h}_t)
\rightarrow
\mathcal{J}_t
\rightarrow
\mathbf{p}_t
\rightarrow
A_t
\rightarrow
(R_t,Z_t).
\]
The learner observes $\mathbf{X}_t$ and $\mathbf{q}_t$, and therefore
knows $\mathbf{h}_t$, but does not observe $\mathcal{J}_t$ before
transmission. The adversary is thus non-anticipating with respect to
both the current policy $\mathbf{p}_t$ and the sampled action $A_t$.

\subsection{Communication Reward and Bandit Feedback}

Let $H_t(k)$, $N_t(k)$, and $O_t(k)\in\{0,1\}$ denote the channel gain,
noise power, and occupancy indicator of channel $k$. Given the realized
attack set, define
\begin{equation}
	\begin{aligned}
		\operatorname{SINR}_t(k)
		&=
		\frac{P_{\mathrm{tx}}H_t(k)}
		{N_t(k)+I_{\mathrm{jam}}
			\mathbf{1}\{k\in\mathcal{J}_t\}},\\
		S_t(k)
		&=
		\mathbf{1}\{O_t(k)=0\}
		\mathbf{1}\{
		\operatorname{SINR}_t(k)\geq\Gamma_{\mathrm{out}}
		\},\\
		r_t(k)
		&=
		S_t(k)
		\frac{
			\min\{
			\log_2(1+\operatorname{SINR}_t(k)),
			R_{\max}
			\}
		}{R_{\max}}
		\in[0,1].
	\end{aligned}
	\label{eq:normalized_reward}
\end{equation}
Thus, reward is zero under occupancy or outage and otherwise follows a
capped spectral-efficiency model.

The corresponding physical goodput is
\begin{equation}
	G_t(k)=W R_{\max}r_t(k),
	\label{eq:goodput_reward_relation}
\end{equation}
where $W$ is the channel bandwidth. Communication-loss regret and
goodput regret therefore differ only by the positive factor
$W R_{\max}$.

The learner observes only the selected-channel feedback
\[
R_t=r_t(A_t),
\qquad
\ell_t(A_t)=1-R_t,
\qquad
Z_t=\mathbf{1}\{A_t\in\mathcal{J}_t\}.
\]
For any policy $\boldsymbol{\pi}\in\Delta_K$, its conditional hit
probability is
\begin{equation}
	\mathcal{R}_t(\boldsymbol{\pi})
	=
	\langle\boldsymbol{\pi},\mathbf{h}_t\rangle.
	\label{eq:policy_hit_risk}
\end{equation}

\subsection{Learning Objectives}

The appropriate contextual structure is unknown. A globally shared
model assumes a common parameter across operating regimes, whereas a
partitioned local model permits region-dependent parameters:
\[
\ell_t(k)
\approx
\langle\boldsymbol{\theta}^{\star},
\mathbf{x}_{t,k}\rangle,
\qquad
\ell_t(k)
\approx
\left\langle
\boldsymbol{\theta}^{\star}_{c(\mathbf{x}_{t,k})},
\boldsymbol{\phi}(\mathbf{x}_{t,k})
\right\rangle.
\]
The global model pools observations efficiently but may be structurally
misspecified; the local model reduces approximation bias at the cost of
fragmented data and greater finite-horizon uncertainty.

D-PACT-AFH addresses three objectives. First, it seeks performance close
to the better fixed base learner on the common execution trajectory.
Second, D-PACT-Hit minimizes communication loss jointly with normalized
excess hit risk:
\begin{equation}
	\sum_{t=1}^{T}
	\left\langle
	\mathbf{p}_t,
	\boldsymbol{\ell}_t
	+
	\beta
	\frac{
		\mathbf{h}_t-r_0\mathbf{1}
	}{
		1-r_0
	}
	\right\rangle,
	\qquad
	r_0=\frac{M}{K},
	\quad
	\beta\geq0.
	\label{eq:composite_objective}
\end{equation}
The normalization measures exposure relative to uniform random hopping.
Third, D-PACT-Safe imposes the per-slot requirement
\begin{equation}
	\langle\mathbf{p}_t,\mathbf{h}_t\rangle
	\leq\tau_t.
	\label{eq:slot_risk_budget}
\end{equation}
Section~\ref{sec:architecture} develops a single online architecture
that realizes these three operating modes.


\section{Online Model Selection for Adaptive Frequency Hopping}
\label{sec:algorithm}
\label{sec:architecture}

D-PACT-AFH combines two structurally distinct contextual learners
through a Tsallis-FTRL master and optionally incorporates predictable
hit risk and a per-slot policy constraint.

\subsection{Global Contextual Base}

The global base is LC-Tsallis-INF with an online covariance estimate
\cite{kato2025lc}. It assumes a shared linear representation across
channels and operating regimes and outputs a policy
$\boldsymbol{\rho}_t^{\mathrm G}\in\Delta_K$ from the public context
$\mathbf{X}_t$. Its loss estimator uses the selected feature
$\mathbf{x}_{t,A_t}$ and an inverse covariance matrix estimated from a
rolling public-context buffer. This shared representation is
statistically efficient when a common linear structure is adequate.

\subsection{Partitioned Local-Linear Base}

The local base uses a public feature subvector
$\mathbf{z}_{t,k}\in[0,1]^{d_{\mathrm L}}$ and a fixed partition
$c:[0,1]^{d_{\mathrm L}}\rightarrow[J]$. Each cell maintains a ridge
model with intercept. Let
$\boldsymbol{\phi}_{t,k}
=[1,\mathbf{z}_{t,k}^{\mathsf T}]^{\mathsf T}$ and
$j=c(\mathbf{z}_{t,k})$. The prediction, uncertainty, optimistic loss,
and resulting policy are
\[
\begin{aligned}
	\widehat{\ell}^{\mathrm L}_t(k)
	&=
	\left[
	\boldsymbol{\phi}_{t,k}^{\mathsf T}
	\mathbf{A}_{t,j}^{-1}\mathbf{b}_{t,j}
	\right]_{[0,1]},\\
	u_t(k)
	&=
	\sqrt{
		\boldsymbol{\phi}_{t,k}^{\mathsf T}
		\mathbf{A}_{t,j}^{-1}
		\boldsymbol{\phi}_{t,k}
	},\\
	\widetilde{\ell}^{\mathrm L}_t(k)
	&=
	\widehat{\ell}^{\mathrm L}_t(k)
	-
	\alpha_{\mathrm L}
	\sqrt{\log(\max\{t,2\})}\,u_t(k),\\
	\rho_t^{\mathrm L}(k)
	&=
	\frac{
		\exp[-\widetilde{\ell}^{\mathrm L}_t(k)/\vartheta_t]
	}{
		\sum_{i=1}^{K}
		\exp[-\widetilde{\ell}^{\mathrm L}_t(i)/\vartheta_t]
	} .
\end{aligned}
\]
Here $[x]_{[0,1]}=\min\{1,\max\{0,x\}\}$,
$\mathbf{A}_{t,j}$ includes ridge regularization, and
$\vartheta_t$ decreases to a positive floor. Partitioning captures
observable heterogeneity, while regularization, optimism, forgetting,
and softmax randomization limit finite-sample instability.

\subsection{Predictable Risk Signal and Model-Class Master}

Let $\mathcal{B}=\{\mathrm G,\mathrm L\}$ and let
$\boldsymbol{\alpha}_t\in\Delta_2$ denote the master distribution over
the two model classes. The predicted exposure and its normalized excess
over uniform random hopping are
\begin{equation}
	\begin{aligned}
		d_t^b
		&=
		\left\langle
		\boldsymbol{\rho}_t^b,\mathbf{h}_t
		\right\rangle,\\
		g_t^b
		&=
		\frac{d_t^b-r_0}{1-r_0},
		\qquad
		r_0=\frac{M}{K},
		\qquad b\in\mathcal{B}.
	\end{aligned}
	\label{eq:base_risk_signal}
\end{equation}
Thus, $g_t^b$ measures the base policy's predicted exposure relative to
the random-hopping reference.

Let $\widehat{L}_{t-1}^b$ and $H_{t-1}^b$ denote the cumulative
estimated communication loss and normalized risk of base $b$. Before
action selection, its master score is
\begin{equation}
	C_t^b
	=
	\widehat{L}_{t-1}^b
	+
	\beta
	\left(
	H_{t-1}^b+\omega_t g_t^b
	\right),
	\label{eq:effective_master_score}
\end{equation}
where $\beta\geq0$ controls hit awareness and $\omega_t$ weights the
current predictable signal. Unlike communication loss, this current
risk term requires no bandit estimation.

For
$\mathbf{C}_t=[C_t^{\mathrm G},C_t^{\mathrm L}]^{\mathsf T}$, the
unanchored Tsallis-FTRL distribution is
\begin{equation}
	\begin{aligned}
		\widetilde{\boldsymbol{\alpha}}_t
		&=
		\argmin_{\boldsymbol{\alpha}\in\Delta_2}
		\left\{
		\langle\boldsymbol{\alpha},\mathbf{C}_t\rangle
		+
		\frac{1}{\eta_t}\Psi(\boldsymbol{\alpha})
		\right\},\\
		\Psi(\boldsymbol{\alpha})
		&=
		-2\sum_{b\in\mathcal{B}}\sqrt{\alpha(b)},
		\qquad
		\eta_t=\frac{\eta_0}{\sqrt{t}} .
	\end{aligned}
	\label{eq:master_tsallis_ftrl}
\end{equation}

To discourage premature migration toward the data-hungry local learner,
D-PACT applies a decaying global anchor:
\begin{equation}
	\begin{aligned}
		\boldsymbol{\alpha}_t
		&=
		(1-a_t)\widetilde{\boldsymbol{\alpha}}_t
		+a_t\mathbf{e}_{\mathrm G},\\
		a_t
		&=
		\left[
		a_{\min}+\frac{a_0}{\sqrt{t}}
		\right]_{[0,1]},
		\qquad
		\mathbf{e}_{\mathrm G}
		=
		[1,0]^{\mathsf T}.
	\end{aligned}
	\label{eq:global_anchor}
\end{equation}
The local base additionally starts with the one-time prior
$\widehat{L}_0^{\mathrm L}=\delta_{\mathrm L}$, whereas
$\widehat{L}_0^{\mathrm G}=0$. The prior contributes only a constant
initial bias; a positive $a_{\min}$ remains a finite-horizon
stabilization term, as quantified in Section~\ref{sec:theory}.

With
$\mathbf{P}_t=
[\boldsymbol{\rho}_t^{\mathrm G},
\boldsymbol{\rho}_t^{\mathrm L}]$,
the model-class mixture is
\begin{equation}
	\boldsymbol{\pi}_t^{\mathrm{mix}}
	=
	\mathbf{P}_t\boldsymbol{\alpha}_t
	=
	\sum_{b\in\mathcal{B}}
	\alpha_t(b)\boldsymbol{\rho}_t^b .
	\label{eq:model_class_mixture}
\end{equation}

\subsection{Coverage-Aware Exploration}

D-PACT directs exploration toward actions receiving substantial mass
from at least one active base. Specifically,
\[
\begin{aligned}
	s_t(k)
	&=
	\sqrt{
		\sum_{b\in\mathcal{B}}
		\alpha_t(b)[\rho_t^b(k)]^2
	},\\
	e_t^{\mathrm{core}}(k)
	&=
	\frac{s_t(k)}{\sum_{i=1}^{K}s_t(i)},\\
	\mathbf{e}_t
	&=
	(1-\mu)\mathbf{e}_t^{\mathrm{core}}
	+
	\mu\frac{\mathbf{1}}{K},
	\qquad \mu\in(0,1].
\end{aligned}
\]
The policy before Safe projection is
\begin{equation}
	\begin{aligned}
		\overline{\boldsymbol{\pi}}_t
		&=
		(1-\gamma_t)\boldsymbol{\pi}_t^{\mathrm{mix}}
		+\gamma_t\mathbf{e}_t,\\
		\gamma_t
		&=
		\min\left\{
		\gamma_{\max},
		\frac{\gamma_0}{\sqrt{t}}
		\right\}.
	\end{aligned}
	\label{eq:preprojection_policy}
\end{equation}
Consequently,
$\overline{\pi}_t(k)\geq\gamma_t\mu/K$ for every channel. The main
configuration applies no additional risk-directed exploration tilt;
the effect of D-PACT-Hit therefore arises from model selection rather
than from a separate modification of the exploration distribution.

\subsection{Safe Policy Projection}

D-PACT-Safe projects the complete pre-projection policy onto the
risk-feasible simplex:
\begin{equation}
	\begin{aligned}
		\boldsymbol{\pi}_t^{\mathrm{safe}}
		=
		\argmin_{\boldsymbol{\pi}\in\Delta_K}
		\quad&
		D_{\mathrm{KL}}
		\left(
		\boldsymbol{\pi}
		\Vert
		\overline{\boldsymbol{\pi}}_t
		\right)\\
		\mathrm{s.t.}\quad&
		\langle\boldsymbol{\pi},\mathbf{h}_t\rangle
		\leq\tau_t .
	\end{aligned}
	\label{eq:safe_projection}
\end{equation}
If the unconstrained policy is feasible, the projection is inactive.
Otherwise, whenever
$\min_k h_t(k)<\tau_t<
\langle\overline{\boldsymbol{\pi}}_t,\mathbf{h}_t\rangle$,
the solution has the exponential form
\begin{equation}
	\pi_t^{\mathrm{safe}}(k)
	=
	\frac{
		\overline{\pi}_t(k)
		\exp[-\lambda_t h_t(k)]
	}{
		\sum_{i=1}^{K}
		\overline{\pi}_t(i)
		\exp[-\lambda_t h_t(i)]
	},
	\label{eq:safe_exponential_tilt}
\end{equation}
where $\lambda_t>0$ is found by bisection. If
$\tau_t<\min_k h_t(k)$, no feasible policy exists; the implementation
uses a distribution supported on minimum-risk channels and records the
unavoidable violation.

The executed policy is
\[
\mathbf{p}_t
=
\begin{cases}
	\overline{\boldsymbol{\pi}}_t,
	& \text{Base or Hit},\\
	\boldsymbol{\pi}_t^{\mathrm{safe}},
	& \text{Safe}.
\end{cases}
\]
D-PACT-Base sets $\beta=0$, D-PACT-Hit uses $\beta>0$ without
projection, and D-PACT-Safe applies \eqref{eq:safe_projection} after the
same risk-aware master.

\subsection{Bandit and Off-Policy Updates}

After sampling $A_t\sim\mathbf{p}_t$, the master evaluates each base
using the unclipped importance-weighted estimate
\begin{equation}
	\widehat{\ell}_t^b
	=
	\frac{
		\rho_t^b(A_t)\ell_t(A_t)
	}{
		p_t(A_t)
	}.
	\label{eq:master_loss_estimator}
\end{equation}
It then updates
$\widehat{L}_t^b=
\widehat{L}_{t-1}^b+\widehat{\ell}_t^b$ and
$H_t^b=H_{t-1}^b+g_t^b$.

The contextual bases instead use the clipped ratio
\begin{equation}
	\widetilde{w}_t^b
	=
	\min\left\{
	\frac{\rho_t^b(A_t)}{p_t(A_t)},
	C_{\mathrm{off}}
	\right\}.
	\label{eq:clipped_base_ratio}
\end{equation}
The global base uses $\widetilde{w}_t^{\mathrm G}$ in its parameter
update. For the local base, let
$j_t=c(\mathbf{z}_{t,A_t})$ and
$\boldsymbol{\phi}_t=\boldsymbol{\phi}_{t,A_t}$. Its data statistics
are updated as
\[
\begin{aligned}
	\mathbf{A}_{t+1,j_t}^{\mathrm{data}}
	&=
	\chi_t\mathbf{A}_{t,j_t}^{\mathrm{data}}
	+
	\widetilde{w}_t^{\mathrm L}
	\boldsymbol{\phi}_t\boldsymbol{\phi}_t^{\mathsf T},\\
	\mathbf{b}_{t+1,j_t}^{\mathrm{data}}
	&=
	\chi_t\mathbf{b}_{t,j_t}^{\mathrm{data}}
	+
	\widetilde{w}_t^{\mathrm L}
	\boldsymbol{\phi}_t\ell_t(A_t),
\end{aligned}
\]
where $\chi_t$ is the elapsed-time forgetting factor. Thus, explicit
exploration preserves common support, clipping stabilizes the base
updates, and the master retains an unbiased estimator whenever the
executed policy has full support.

\subsection{Complete Procedure}

Algorithm~\ref{alg:dpact_afh} summarizes the three operating modes.

\begin{algorithm}[!t]
	\caption{D-PACT-AFH}
	\label{alg:dpact_afh}
	\footnotesize
	\begin{algorithmic}[1]
		\REQUIRE Partition $c$; schedules $\eta_t,\gamma_t,a_t$;
		risk weight $\beta$; ratio cap $C_{\mathrm{off}}$;
		budget $\tau_t$; mode $m\in\{\mathrm{Base},\mathrm{Hit},\mathrm{Safe}\}$
		\STATE Initialize both bases,
		$\widehat{L}_0^{\mathrm G}=0$,
		$\widehat{L}_0^{\mathrm L}=\delta_{\mathrm L}$,
		and $H_0^b=0$
		\FOR{$t=1,\ldots,T$}
		\STATE Observe $\mathbf{X}_t$ and $\mathbf{q}_t$; derive
		$\mathbf{h}_t$
		\STATE Obtain
		$\boldsymbol{\rho}_t^{\mathrm G}$ and
		$\boldsymbol{\rho}_t^{\mathrm L}$; compute $g_t^b$
		\STATE Form $\boldsymbol{\alpha}_t$ using
		\eqref{eq:effective_master_score}--\eqref{eq:global_anchor}
		\STATE Form
		$\boldsymbol{\pi}_t^{\mathrm{mix}}$ and
		$\overline{\boldsymbol{\pi}}_t$
		\IF{$m=\mathrm{Safe}$}
		\STATE Project $\overline{\boldsymbol{\pi}}_t$ using
		\eqref{eq:safe_projection}
		\ELSE
		\STATE Set
		$\mathbf{p}_t=\overline{\boldsymbol{\pi}}_t$
		\ENDIF
		\STATE Sample $A_t\sim\mathbf{p}_t$ and observe $(R_t,Z_t)$
		\FOR{$b\in\{\mathrm G,\mathrm L\}$}
		\STATE Update the master using
		\eqref{eq:master_loss_estimator}
		\STATE Update base $b$ using
		\eqref{eq:clipped_base_ratio}; update $H_t^b$
		\ENDFOR
		\ENDFOR
	\end{algorithmic}
\end{algorithm}


\section{Performance Analysis}
\label{sec:performance_analysis}
\label{sec:theory}

We analyze estimator validity, model-selection overhead, and the effects
of anchoring, clipping, exploration, Safe projection, and fallback.
The primary comparator is the better continuously updated base on the
common D-PACT-AFH trajectory; comparison with independently executed
bases requires an additional transfer term. Detailed derivations are
deferred to the appendices.

\subsection{Information Structure and Comparator}

Let $\mathcal{P}_t$ contain the public information available after the
executed policy $\mathbf{p}_t$ is formed but before $A_t$ is sampled. It
includes the interaction history, current public contexts and attack
marginals, both base policies, and the master weights, but excludes the
hidden attack set $\mathcal{J}_t$ and sampled action $A_t$. Under the
non-anticipating attack model,
\[
\Pr(k\in\mathcal{J}_t\mid\mathcal{P}_t)=h_t(k),
\qquad
A_t\sim\mathbf{p}_t,
\qquad
A_t\perp\!\!\!\perp\mathcal{J}_t\mid\mathcal{P}_t .
\]

For estimator analysis, let $\mathcal{G}_t$ augment $\mathcal{P}_t$
with the committed attack set, channel state, and full loss vector, but
not $A_t$, and write
$\mathbb{E}_t[\cdot]=\mathbb{E}[\cdot\mid\mathcal{G}_t]$.

With $r_0=M/K<1$, define the scalarized action and base costs
\[
\begin{aligned}
	\kappa_t(k)
	&=
	\ell_t(k)
	+
	\beta\frac{h_t(k)-r_0}{1-r_0},\\
	c_t^b
	&=
	\left\langle
	\boldsymbol{\rho}_t^b,
	\boldsymbol{\kappa}_t
	\right\rangle
	=
	\ell_t^b+\beta g_t^b,
	\qquad
	\ell_t^b
	=
	\langle
	\boldsymbol{\rho}_t^b,
	\boldsymbol{\ell}_t
	\rangle .
\end{aligned}
\]
Their per-slot span is bounded by
\[
\max_k\kappa_t(k)-\min_k\kappa_t(k)
\leq
G_c
\triangleq
1+\frac{\beta}{1-r_0}.
\]

The primary comparator is
\[
b^\star
\in
\argmin_{b\in\mathcal{B}} C_T^b,
\qquad
C_T^b
=
\mathbb{E}
\left[
\sum_{t=1}^{T}c_t^b
\right],
\]
where both shadow bases are continuously updated from the same executed
trajectory. This comparator is not an independently executed
counterfactual base.

\subsection{Conditional Estimation and Goodput}

The master estimator is given by
\eqref{eq:master_loss_estimator}.

\begin{lemma}[Conditional estimation]
	\label{lem:pa_unbiasedness}
	Let
	$\mathcal{S}_t=\{k:p_t(k)>0\}$. Then
	\[
	\mathbb{E}_t[\widehat{\ell}_t^b]
	=
	\ell_t^b-B_{t,\mathrm{fb}}^b,
	\qquad
	B_{t,\mathrm{fb}}^b
	=
	\sum_{k\notin\mathcal{S}_t}
	\rho_t^b(k)\ell_t(k)
	\geq0.
	\]
	In particular, the estimator is conditionally unbiased whenever
	$p_t(k)>0$ for every $k$ with $\rho_t^b(k)>0$.
\end{lemma}

\begin{IEEEproof}
	Conditioned on $\mathcal{G}_t$, only $A_t$ is random. Summing the
	importance-weighted estimator over $\mathcal{S}_t$ recovers the base
	loss on that support; the omitted probability mass is exactly
	$B_{t,\mathrm{fb}}^b$.
\end{IEEEproof}

Hence, an attacker may depend on all past interactions and current
public contexts without invalidating the estimator, provided that it
does not react to the current sampled action.

Moreover, because
$\mathbf{G}_t=WR_{\max}(\mathbf{1}-\boldsymbol{\ell}_t)$,
communication-loss and goodput differences between any two policies
are identical up to the positive factor $WR_{\max}$. The loss analysis
therefore applies directly to the reported goodput metric.

\subsection{Risk-Aware Tsallis Model Selection}

Define the predictable hint, estimated scalarized cost, and residual by
\[
\begin{aligned}
	m_t^b
	&=
	\beta\omega_tg_t^b,\\
	\widehat{c}_t^b
	&=
	\widehat{\ell}_t^b+\beta g_t^b,\\
	z_t^b
	&=
	\widehat{c}_t^b-m_t^b
	=
	\widehat{\ell}_t^b
	+
	\beta(1-\omega_t)g_t^b .
\end{aligned}
\]
Thus, when $\omega_t=1$, the current visible risk is entirely contained
in the predictable hint.

For the unanchored master
$\widetilde{\boldsymbol{\alpha}}_t$, let
\[
V_t
=
\sum_{b\in\mathcal{B}}
[\widetilde{\alpha}_t(b)]^{3/2}(z_t^b)^2,
\qquad
D_\Psi=2(\sqrt{2}-1).
\]

\begin{theorem}[Risk-aware Tsallis master]
	\label{thm:pa_master}
	Suppose $\eta_t$ is nonincreasing and the Tsallis-$1/2$ local-stability
	condition
	\[
	\eta_t\|\mathbf{z}_t\|_{\ast,t}
	\leq c_0<1,
	\qquad
	\|\mathbf{z}\|_{\ast,t}^2
	=
	2\sum_b
	[\widetilde{\alpha}_t(b)]^{3/2}z(b)^2
	\]
	holds. Then, for every fixed
	$\boldsymbol{u}\in\Delta_2$,
	\begin{equation}
		\begin{aligned}
			\mathbb{E}
			\left[
			\sum_{t=1}^{T}
			\left\langle
			\widetilde{\boldsymbol{\alpha}}_t-\boldsymbol{u},
			\mathbf{c}_t
			\right\rangle
			\right]
			\leq
			\frac{D_\Psi}{\eta_T}
			+
			C_\Psi
			\sum_{t=1}^{T}
			\eta_t\mathbb{E}[V_t]
			+
			E_T^{\mathrm{fb}},
		\end{aligned}
		\label{eq:pa_master_bound}
	\end{equation}
	where $C_\Psi$ depends only on $c_0$ and
	\[
	E_T^{\mathrm{fb}}
	=
	\sum_{t=1}^{T}
	\max_{b\in\mathcal{B}}
	\mathbb{E}[B_{t,\mathrm{fb}}^b].
	\]
	The final term vanishes when full support is maintained.
\end{theorem}

\begin{IEEEproof}[Proof sketch]
	The optimistic-FTRL stability--penalty decomposition bounds the
	regularization penalty by $D_\Psi/\eta_T$. The inverse Hessian of the
	Tsallis regularizer induces the stated local norm, yielding the
	variance term
	$C_\Psi\sum_t\eta_tV_t$. Lemma~\ref{lem:pa_unbiasedness} converts the
	estimated losses to conditional base losses, with
	$E_T^{\mathrm{fb}}$ accounting for any lost support. A complete proof is
	given in Appendix~\ref{app:master_proof}.
\end{IEEEproof}

This result uses the same convex-analytic backbone as LC-Tsallis-INF
\cite{kato2025lc}, but it does not inherit its final regret rate
unchanged: the Global base uses an empirical covariance estimate, the
Local base is structurally distinct, and the outer algorithm includes
anchoring, clipping, exploration, and Safe execution.

Anchoring changes the cumulative master cost by at most
\[
E_T^{\mathrm{anchor}}
=
G_c\sum_{t=1}^{T}a_t
\leq
G_c
\left(
a_{\min}T+2a_0\sqrt{T}
\right),
\]
while the one-time Local prior contributes at most
$E_T^{\mathrm{prior}}=\delta_{\mathrm L}$ when the Local base is the
comparator. A positive $a_{\min}$ is therefore a finite-horizon
stabilizer rather than an asymptotically no-regret choice.

\subsection{Base Updates and Execution Remainders}

The master remains unclipped, whereas each contextual base uses the
ratio in \eqref{eq:clipped_base_ratio}. Let
\[
W_t^b
=
\frac{\rho_t^b(A_t)}{p_t(A_t)},
\qquad
X_{t,C}^b
=
\min\{W_t^b,C_{\mathrm{off}}\}\ell_t(A_t).
\]

\begin{proposition}[Clipping bias and stability]
	\label{prop:pa_clipping}
	On a full-support round,
	\[
	\begin{aligned}
		\ell_t^b-\mathbb{E}_t[X_{t,C}^b]
		&=
		\mathbb{E}_t
		\left[
		(W_t^b-C_{\mathrm{off}})_+
		\ell_t(A_t)
		\right]
		\geq0,\\
		\mathbb{E}_t[(X_{t,C}^b)^2]
		&\leq
		C_{\mathrm{off}}\ell_t^b
		\leq C_{\mathrm{off}}.
	\end{aligned}
	\]
	Thus, clipping bounds the second moment at the cost of downward bias in
	the base update.
\end{proposition}

\begin{IEEEproof}
	Subtract the clipped estimator from its unclipped counterpart and use
	$\min\{W,C\}^2\leq C\min\{W,C\}\leq CW$ together with
	$\ell_t(A_t)^2\leq\ell_t(A_t)$.
\end{IEEEproof}

Clipping does not alter Theorem~\ref{thm:pa_master} directly because
the master uses the unclipped estimator. It instead affects future base
recommendations. For comparison with ideal on-policy or independently
executed bases, we collect these effects as
\[
\begin{aligned}
	E_T^{\mathrm G}
	&=
	E_T^{\mathrm{clip,G}}
	+
	E_T^{\mathrm{cov}},\\
	E_T^{\mathrm L}
	&=
	E_T^{\mathrm{clip,L}}
	+
	E_T^{\mathrm{forget}}
	+
	E_T^{\mathrm{approx}}.
\end{aligned}
\]

Coverage-aware exploration contributes at most
\[
\left|
\left\langle
\overline{\boldsymbol{\pi}}_t
-
\boldsymbol{\pi}_t^{\mathrm{mix}},
\boldsymbol{\kappa}_t
\right\rangle
\right|
\leq
G_c\gamma_t,
\]
and therefore
$E_T^{\mathrm{explore}}
=G_c\sum_{t=1}^{T}\gamma_t$.

\subsection{Safe Execution and Hit-Risk Control}

The Safe policy is the KL projection defined in
\eqref{eq:safe_projection}.

\begin{proposition}[Safe projection]
	\label{prop:pa_safe_projection}
	The Safe problem is feasible if and only if
	$\tau_t\geq h_{t,\min}\triangleq\min_k h_t(k)$, and its feasible
	solution is unique. The projection is inactive when the unprojected
	policy already satisfies the budget. For
	\[
	h_{t,\min}
	<
	\tau_t
	<
	\langle
	\overline{\boldsymbol{\pi}}_t,
	\mathbf{h}_t
	\rangle,
	\]
	the constraint is active and the unique solution is the exponential
	tilt in \eqref{eq:safe_exponential_tilt}. At
	$\tau_t=h_{t,\min}$, it is the renormalization of
	$\overline{\boldsymbol{\pi}}_t$ over the minimum-risk channels.
	
	Moreover, on every feasible Safe round,
	\begin{equation}
		\begin{aligned}
			\left|
			\left\langle
			\boldsymbol{\pi}_t^{\mathrm{safe}}
			-
			\overline{\boldsymbol{\pi}}_t,
			\boldsymbol{\kappa}_t
			\right\rangle
			\right|
			\leq
			G_c
			\sqrt{
				\frac{1}{2}
				D_{\mathrm{KL}}
				\left(
				\boldsymbol{\pi}_t^{\mathrm{safe}}
				\Vert
				\overline{\boldsymbol{\pi}}_t
				\right)
			}.
		\end{aligned}
		\label{eq:pa_safe_distortion}
	\end{equation}
\end{proposition}

\begin{IEEEproof}[Proof sketch]
	Feasibility follows because the minimum achievable policy risk is
	$h_{t,\min}$. Strict convexity gives uniqueness, and the KKT conditions
	give the exponential tilt. The distortion bound follows from the
	total-variation span inequality and Pinsker's inequality. A complete proof is given in Appendix~\ref{app:safe_oracle}.
\end{IEEEproof}

\begin{theorem}[Exact conditional hit-risk control]
	\label{thm:pa_exact_risk}
	Under the fixed-budget attack model, every feasible Safe policy
	satisfies
	\begin{equation}
		\Pr(
		A_t\in\mathcal{J}_t
		\mid\mathcal{P}_t
		)
		=
		\left\langle
		\boldsymbol{\pi}_t^{\mathrm{safe}},
		\mathbf{h}_t
		\right\rangle
		\leq\tau_t.
		\label{eq:pa_exact_hit_bound}
	\end{equation}
	If $\tau_t<h_{t,\min}$, every policy violates the budget by at least
	$h_{t,\min}-\tau_t$, and a minimum-risk fallback attains this smallest
	possible violation.
\end{theorem}

\begin{IEEEproof}
	Conditional independence gives
	\[
	\Pr(A_t\in\mathcal{J}_t\mid\mathcal{P}_t)
	=
	\sum_k
	p_t(k)
	\Pr(k\in\mathcal{J}_t\mid\mathcal{P}_t)
	=
	\langle\mathbf{p}_t,\mathbf{h}_t\rangle.
	\]
	The claim then follows from feasibility and the definition of
	$h_{t,\min}$.
\end{IEEEproof}

\begin{corollary}[Risk-model misspecification]
	\label{cor:pa_risk_misspecification}
	If Safe uses $\widetilde{\mathbf{h}}_t$ satisfying
	$\|\widetilde{\mathbf{h}}_t-\mathbf{h}_t\|_\infty
	\leq\varepsilon_t$, then
	\[
	\langle
	\mathbf{p}_t,
	\widetilde{\mathbf{h}}_t
	\rangle
	\leq\tau_t
	\quad\Longrightarrow\quad
	\Pr(A_t\in\mathcal{J}_t\mid\mathcal{P}_t)
	\leq\tau_t+\varepsilon_t.
	\]
\end{corollary}

\subsection{Overall Oracle Guarantee}

Define the implementation remainders
\[
\begin{aligned}
	E_T^{\mathrm{anchor}}
	&=
	G_c\sum_{t=1}^{T}a_t,
	&
	E_T^{\mathrm{prior}}
	&=
	\delta_{\mathrm L},\\
	E_T^{\mathrm{explore}}
	&=
	G_c\sum_{t=1}^{T}\gamma_t,
	&
	E_T^{\mathrm{safe}}
	&=
	G_c
	\sum_{t\in\mathcal{T}_{\mathrm{safe}}}
	\sqrt{
		\frac{1}{2}
		D_{\mathrm{KL}}
		(
		\mathbf{p}_t
		\Vert
		\overline{\boldsymbol{\pi}}_t
		)
	}.
\end{aligned}
\]
Let $E_T^{\mathrm{fallback}}$ collect the support-loss term
$E_T^{\mathrm{fb}}$ and the unavoidable execution cost on infeasible
fallback rounds.

\begin{theorem}[D-PACT-AFH oracle decomposition]
	\label{thm:pa_oracle}
	Under the conditions of Theorem~\ref{thm:pa_master},
	\begin{equation}
		\begin{aligned}
			&
			\mathbb{E}
			\left[
			\sum_{t=1}^{T}
			\left\langle
			\mathbf{p}_t,
			\boldsymbol{\kappa}_t
			\right\rangle
			\right]
			-
			\min_{b\in\mathcal{B}}C_T^b\\
			&\quad\leq
			\frac{D_\Psi}{\eta_T}
			+
			C_\Psi
			\sum_{t=1}^{T}
			\eta_t\mathbb{E}[V_t]
			+
			E_T^{\mathrm{anchor}}
			+
			E_T^{\mathrm{prior}}\\
			&\qquad\quad+
			E_T^{\mathrm{explore}}
			+
			E_T^{\mathrm{safe}}
			+
			E_T^{\mathrm{fallback}}.
		\end{aligned}
		\label{eq:pa_oracle_shadow}
	\end{equation}
	The minimum is over the Global and Local shadow bases updated on the
	common execution trajectory.
	
	For independently executed references with costs $C_T^{b,\circ}$, the
	same bound holds with the additional transfer term
	\[
	E_T^{\mathrm{transfer}}
	=
	\max_{b\in\mathcal{B}}
	[C_T^b-C_T^{b,\circ}]_+,
	\]
	which may be decomposed through $E_T^{\mathrm G}$ and
	$E_T^{\mathrm L}$.
\end{theorem}

\begin{IEEEproof}[Proof sketch]
	Apply Theorem~\ref{thm:pa_master} to the unanchored master and add the
	anchor and prior costs. The deviations between the master mixture and
	the executed policy are bounded by the exploration and Safe-distortion
	terms; fallback rounds contribute
	$E_T^{\mathrm{fallback}}$. The standalone comparison follows by adding
	the discrepancy between each shadow base and its independently executed
	reference. A complete proof is given in Appendix~\ref{app:safe_oracle}.
\end{IEEEproof}

\begin{corollary}[Sublinear special case]
	\label{cor:pa_sublinear}
	Suppose
	$\eta_t=\Theta(t^{-1/2})$,
	$\sup_t\mathbb{E}[V_t]<\infty$,
	$a_{\min}=0$,
	$a_t=O(t^{-1/2})$, and
	$\gamma_t=O(t^{-1/2})$. If
	\[
	E_T^{\mathrm{safe}}
	+
	E_T^{\mathrm{fallback}}
	+
	E_T^{\mathrm{transfer}}
	=
	o(T),
	\]
	then the average excess cost relative to the better independent base
	converges to zero. If these three terms are
	$\widetilde{O}(\sqrt{T})$, the total excess cost is
	$\widetilde{O}(\sqrt{T})$.
\end{corollary}

A positive anchor floor $a_{\min}$ instead contributes the explicit
linear term $G_ca_{\min}T$ and must be retained when interpreting the
finite-horizon implementation.


\section{Numerical Results}
\label{sec:experiments}

We evaluate D-PACT-AFH in terms of communication utility,
predictive-jamming exposure, model-class adaptation, and computational
cost. Unless stated otherwise, each result is averaged over 20
independent seeds. Methods within the same comparison use identical
seed lists and therefore experience the same channel, occupancy, and
attack realizations. Confidence intervals are computed as
$1.96s/\sqrt{n}$, where $s$ is the sample standard deviation across
seeds.

\subsection{Setup and Evaluation Protocol}

The default system contains $K=12$ candidate channels and operates for
$T=10^4$ slots, with one attacked channel per slot. The normalized
reward follows \eqref{eq:normalized_reward}, and physical goodput follows
\eqref{eq:goodput_reward_relation}. Table~\ref{tab:simulation_setup}
lists the parameters most relevant to reproduction.

\begin{table}[!t]
	\centering
	\caption{Default simulation configuration.}
	\label{tab:simulation_setup}
	\small
	\setlength{\tabcolsep}{3pt}
	\begin{tabular}{@{}>{\raggedright\arraybackslash}p{0.62\columnwidth}
		>{\raggedleft\arraybackslash}p{0.32\columnwidth}@{}}
		\toprule
		Parameter & Value \\
		\midrule
		Channels, horizon, attack budget $(K,T,M)$
		& $(12,10^4,1)$ \\
		Occupancy probability
		& $0.15$ \\
		Transmit/noise/jamming power
		& $1/0.05/2$ \\
		Bandwidth and rate cap
		& $1$ MHz, $4$ bit/s/Hz \\
		Outage threshold
		& $3$ dB \\
		Context dimension
		& $4$ \\
		Local bins and ridge coefficient
		& $[4,4,2]$, $0.1$ \\
		Risk and exploration parameters
		& $\beta=16$, $\nu=0$, $\mu=0.02$ \\
		Off-policy ratio cap
		& $C_{\mathrm{off}}=20$ \\
		Safe95 budget
		& $\tau=0.11$ \\
		\bottomrule
	\end{tabular}
\end{table}

The baselines are UCB, Thompson sampling, EXP3,
LC-Tsallis-INF-Online, risk-aware EXP4, and AUFH-EXP3++.
The proposed modes are D-PACT-Base, D-PACT-Hit, and
D-PACT-Safe95. Since $\nu=0$, the improvement of D-PACT-Hit is
attributable to risk-aware model selection rather than a separate
risk-directed exploration tilt.

Safe95 is selected using an independent 20-seed calibration scan.
Among Pareto-efficient budgets retaining at least $95\%$ of the
unconstrained D-PACT-Hit goodput, we select the point with the smallest
expected hit probability. This yields $\tau=0.11$ for $K=12$ and
$M=1$. The calibration and evaluation seed sets are disjoint.

Experiments were conducted in MATLAB R2020b on a 12th Gen Intel
Core i7-12650H processor with 16 GB DDR4 memory. Independent seeds were
executed using 16 MATLAB workers. Runtime excludes environment
generation, aggregation, plotting, and file output and is used only to
compare relative computational overhead.

\subsection{End-to-End Goodput--Risk Performance}

We first consider the contextual white-box jammer, which produces the
strongest policy-dependent exposure among the evaluated attacks.
Fig.~\ref{fig:endpoint_tradeoff} compares the final goodput and
empirical hit probability.

\begin{figure}[!t]
	\centering
	\includegraphics[width=\columnwidth]
	{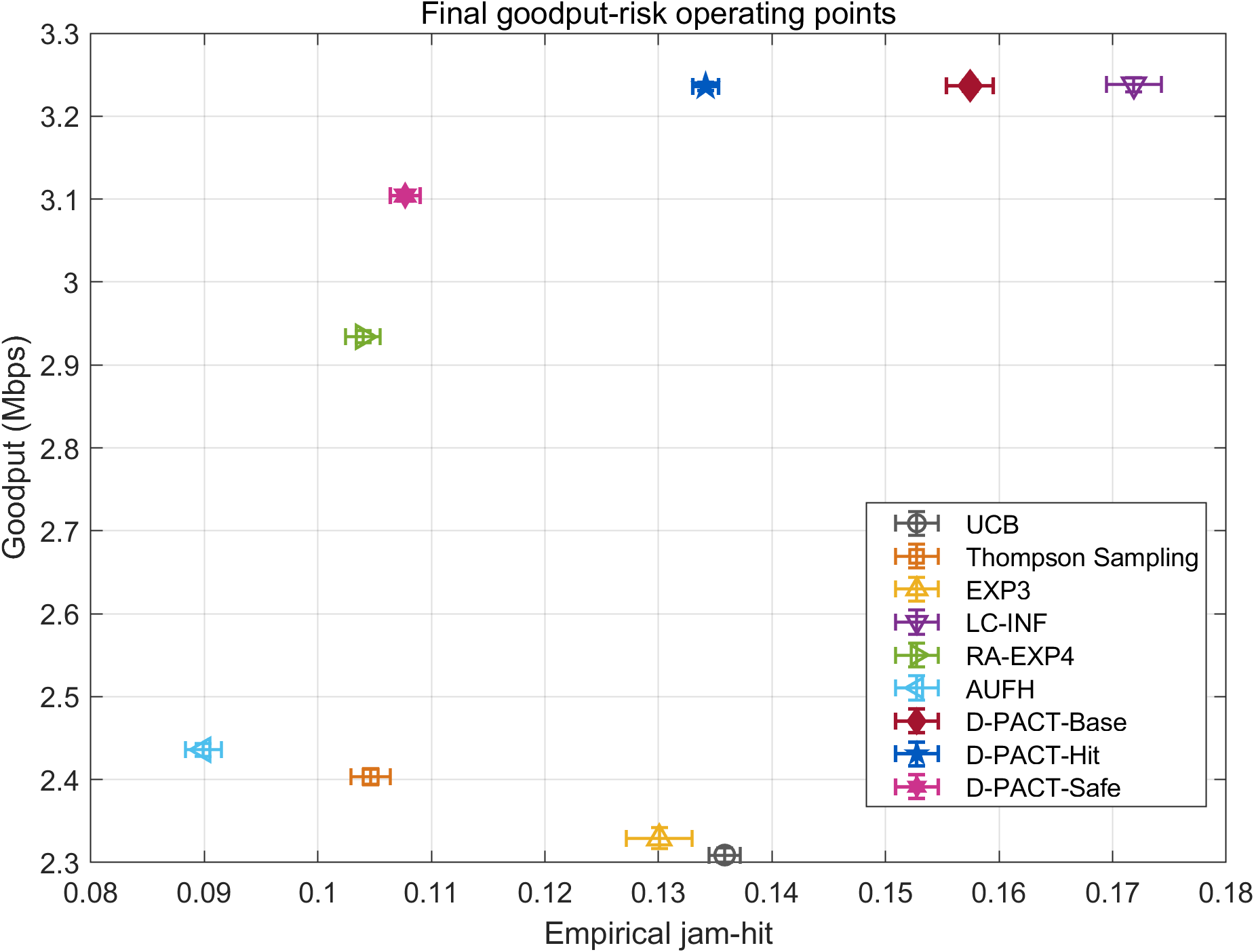}
	\caption{End-to-end goodput--risk operating points under the
		contextual white-box jammer. Error bars denote 95\% confidence
		intervals.}
	\label{fig:endpoint_tradeoff}
\end{figure}

D-PACT-Base obtains $3.236$ Mbps, essentially matching the
$3.238$ Mbps of LC-Tsallis-INF, while reducing the empirical hit rate
from $0.172$ to $0.157$. D-PACT-Hit retains $3.236$ Mbps and further
reduces the hit rate to $0.134$, corresponding to reductions of
$14.8\%$ relative to D-PACT-Base and $21.9\%$ relative to
LC-Tsallis-INF.

D-PACT-Safe95 lowers the hit rate to $0.108$ while retaining
$95.9\%$ of the D-PACT-Hit goodput. Risk-aware EXP4 and AUFH-EXP3++
attain similarly low exposure, but their goodputs are lower by
approximately $0.17$ and $0.67$ Mbps, respectively. D-PACT-Safe95
therefore occupies a more favorable goodput--risk region.

\begin{figure*}[!t]
	\centering
	\includegraphics[width=0.88\textwidth]
	{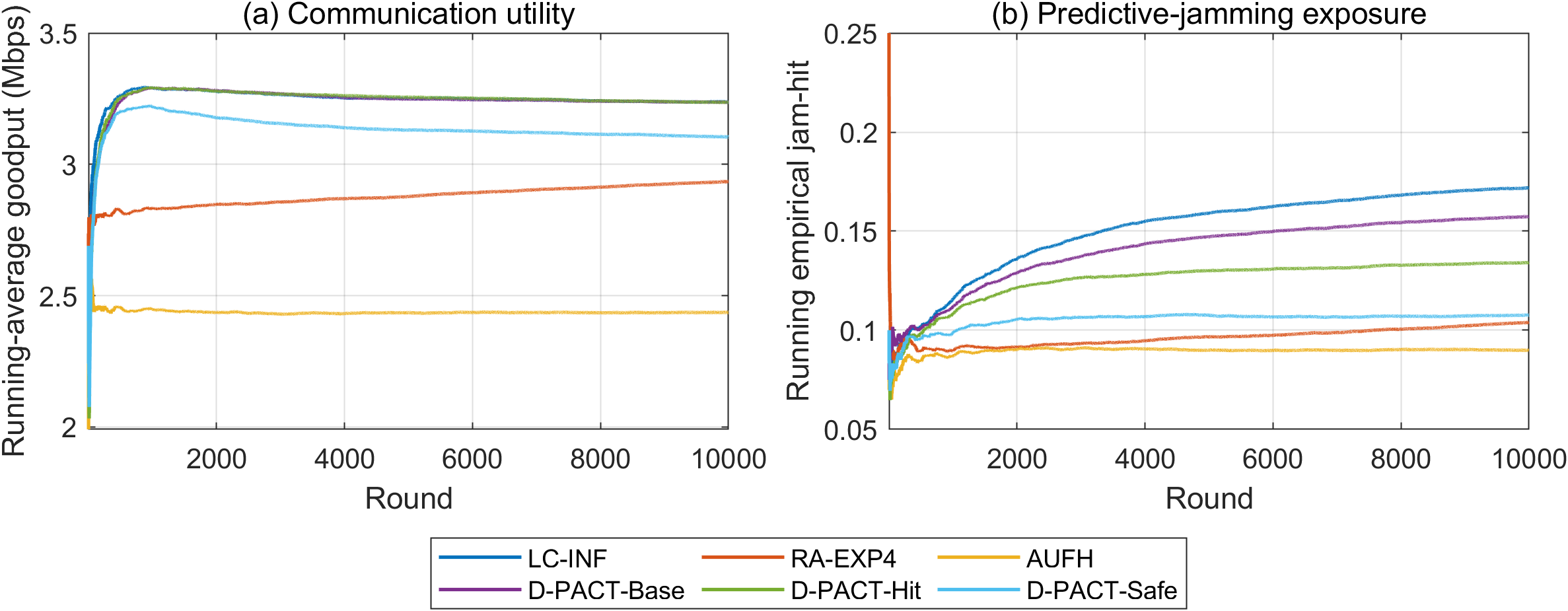}
	\caption{Online communication and exposure performance under the
		contextual white-box jammer. (a) Running-average delivered goodput.
		(b) Running empirical jam-hit probability. Curves are averaged over
		20 independent seeds.}
	\label{fig:running_performance}
\end{figure*}

Fig.~\ref{fig:running_performance} shows the online evolution under the
contextual white-box jammer. D-PACT-Hit approaches the highest
running-average goodput while maintaining substantially lower exposure
than LC-Tsallis-INF and D-PACT-Base. D-PACT-Safe95 further suppresses
the empirical hit probability with a stable and controlled goodput
penalty. The trajectories confirm that the endpoint improvements are
not caused by isolated late-round fluctuations.


\subsection{Mechanism Validation and Model Adaptation}

Table~\ref{tab:ablation} separates the effects of the Global and Local
bases, risk-aware model selection, and Safe projection.

\begin{table}[!t]
	\centering
	\caption{Mechanism ablation under the white-box jammer. Values are
		mean $\pm$ 95\% confidence-interval half-width.}
	\label{tab:ablation}
	\small
	\setlength{\tabcolsep}{3pt}
	\begin{tabular}{lccc}
		\toprule
		Method
		& Goodput
		& Jam hit
		& Local mass \\
		& (Mbps) & & \\
		\midrule
		Global-only
		& $3.235\!\pm\!0.006$
		& $0.172\!\pm\!0.002$
		& $0$ \\
		Local-only
		& $3.239\!\pm\!0.005$
		& $0.133\!\pm\!0.001$
		& $1$ \\
		D-PACT-Base
		& $3.232\!\pm\!0.007$
		& $0.158\!\pm\!0.002$
		& $0.353$ \\
		D-PACT-Hit
		& $3.236\!\pm\!0.005$
		& $0.134\!\pm\!0.001$
		& $0.948$ \\
		D-PACT-Safe95
		& $3.104\!\pm\!0.003$
		& $0.108\!\pm\!0.001$
		& $0.944$ \\
		\bottomrule
	\end{tabular}
\end{table}

Global-only and Local-only achieve similar goodput but differ markedly
in exposure. Without the risk term, D-PACT-Base assigns $0.353$ mean
mass to Local and exhibits an intermediate hit rate. D-PACT-Hit shifts
the Local mass to $0.948$, matching the lower-exposure Local policy
without sacrificing goodput. Safe95 leaves the master allocation nearly
unchanged; its additional reduction therefore arises from policy
projection rather than a second model-selection effect.

\begin{figure*}[!t]
	\centering
	\includegraphics[width=0.88\textwidth]
	{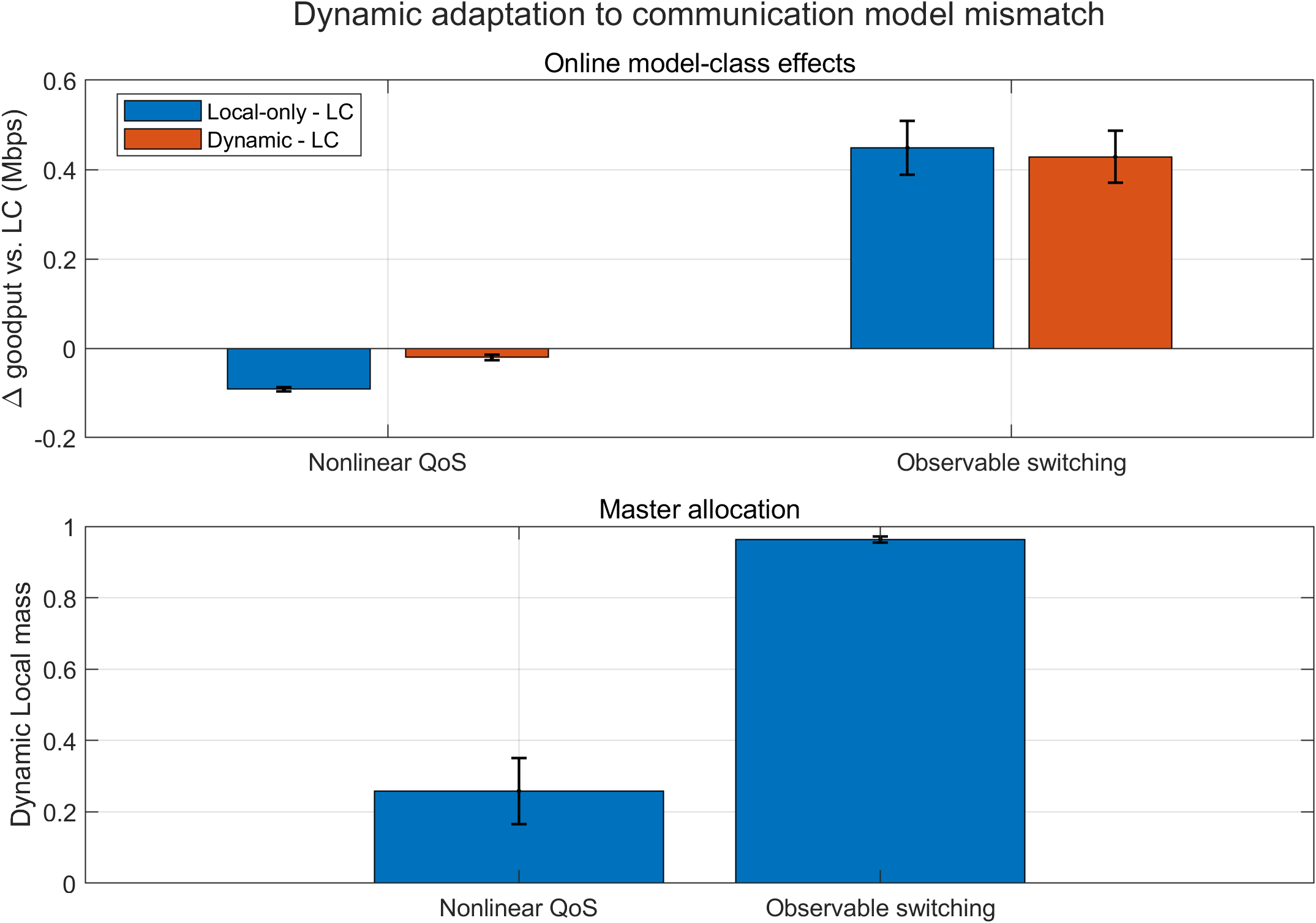}
	\caption{Model adaptation under a nonlinear negative control and an
		observable-switching regime. The upper panel reports goodput
		relative to the Global base; the lower panel reports the mean Local
		master mass.}
	\label{fig:model_selection}
\end{figure*}


Fig.~\ref{fig:model_selection} tests whether the master adapts to model
mismatch rather than systematically favoring the more complex base.
Under observable switching, Local-only improves over Global by
$0.449$ Mbps, and D-PACT recovers $0.428$ Mbps, or $95.5\%$, of this
gain while assigning $0.964$ mass to Local. In the nonlinear negative
control, Local-only is $0.092$ Mbps worse than Global, whereas D-PACT
loses only $0.020$ Mbps and assigns $0.258$ mass to Local. It therefore
avoids $77.7\%$ of the degradation caused by committing to Local.

\subsection{Robustness and Safe Risk Control}

We first examine how the explicit risk budget controls the operating
point of D-PACT-Safe.

\begin{figure}[!t]
	\centering
	\includegraphics[width=0.94\columnwidth]
	{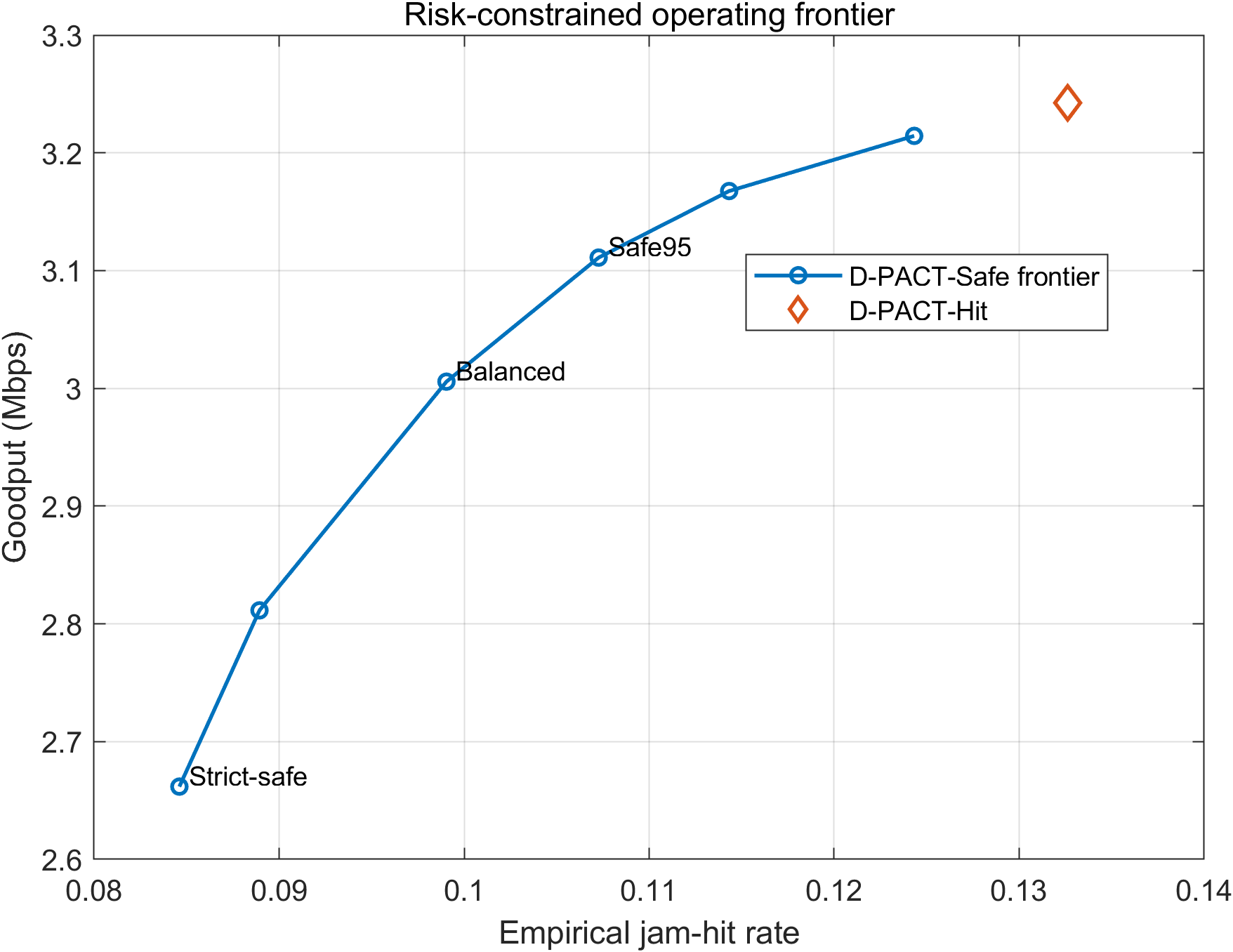}
	\caption{Goodput--risk frontier of D-PACT-Safe under the contextual
		white-box jammer. Safe95 is independently calibrated at
		$\tau=0.11$.}
	\label{fig:safe_frontier}
\end{figure}

Fig.~\ref{fig:safe_frontier} shows a monotone empirical tradeoff.
Tightening $\tau$ from $0.14$ to $0.085$ reduces the empirical hit rate
from $0.124$ to $0.0846$, while goodput decreases from $3.214$ to
$2.662$ Mbps. Safe95 yields $3.111$ Mbps at a hit rate of $0.1073$.
The maximum numerical budget violation is below $10^{-10}$, confirming
that the projection enforces the prescribed constraint to numerical
precision.

We next evaluate random, sweeping, FOLPETTI-inspired, and contextual
white-box jammers. Under random jamming, D-PACT-Base, D-PACT-Hit, and
D-PACT-Safe95 coincide at $3.580$ Mbps and a hit rate of $0.0834$,
close to the uniform baseline $M/K=1/12$. Thus, the risk-aware
mechanisms create no artificial advantage when the public risk is
action-independent.

Across sweeping and FOLPETTI-inspired attacks, all D-PACT variants
retain high communication utility. Under contextual white-box jamming,
D-PACT-Hit and D-PACT-Safe95 reduce the hit rate to $0.135$ and
$0.108$, respectively. Additional tests over nominal, nonlinear,
observable-switching, contaminated, exogenous-sweep, mixed, and
hidden-Markov regimes exhibit the same qualitative behavior. Complete
method-by-attacker comparisons and regime results are reported in
Appendix~\ref{app:cross_attacker} and
Appendix~\ref{app:regime_robustness}.

\subsection{Scaling and Computational Cost}

Scaling experiments use 10 seeds per configuration. As $K$ increases
from $8$ to $32$, D-PACT-Hit goodput rises from $2.931$ to
$3.639$ Mbps, while its empirical hit rate falls from $0.199$ to
$0.0467$. D-PACT-Safe follows the correspondingly scaled risk budgets;
at $K=32$, the unconstrained policy is already feasible, and Safe and
Hit become nearly identical.

Across horizons from $10^3$ to $2\times10^4$, D-PACT-Hit maintains
$3.22$--$3.31$ Mbps and Safe95 maintains $3.08$--$3.23$ Mbps. Runtime
grows approximately linearly with the horizon. Hit adds limited
overhead over Base, whereas Safe incurs the cost of a one-dimensional
bisection per slot. Complete scaling curves and timing results are reported
in Appendix~\ref{app:scaling_runtime}.


\section{Conclusion}
\label{sec:conclusion}

This paper developed D-PACT-AFH for adaptive frequency hopping under
predictive jamming and contextual-model uncertainty. A Tsallis-FTRL
master combines globally shared and locally specialized learners,
allowing the transmitter to adapt its model class from bandit feedback.
D-PACT-Hit incorporates channel-wise marginal attack risk into model
selection, while D-PACT-Safe enforces a per-slot hit-risk budget through
a minimum-KL projection. The analysis established conditional estimator
validity under non-anticipating attacks, an oracle decomposition relative
to the better continuously updated base, and exact conditional hit-risk
control under the fixed-budget attack model.

Numerical results showed that the master adapts to observable model
mismatch without systematically favoring the more complex local learner.
The hit-aware mode reduced predictive exposure with negligible
communication loss, and the Safe mode produced a controllable
goodput--risk frontier across channel regimes and jammer types.

The present guarantees rely on a non-anticipating attacker and accurate
channel-wise risk marginals; misspecification is covered only through a
bounded-error extension. Future work will consider learned or delayed
risk predictors, time-varying attack budgets, interacting transmitter
pairs, and evaluation on software-defined radio platforms.

\appendices

\appendices

\section{Proof of the Risk-Aware Tsallis Master Bound}
\label{app:master_proof}

Let
\[
\widehat{\mathbf C}_{t-1}
=
\sum_{s=1}^{t-1}\widehat{\mathbf c}_s,
\qquad
\widehat{c}_t^b
=
\widehat{\ell}_t^b+\beta g_t^b,
\]
and define
\[
\Phi_t(\boldsymbol{\alpha})
=
\frac{1}{\eta_t}\Psi(\boldsymbol{\alpha}),
\qquad
\Psi(\boldsymbol{\alpha})
=
-2\sum_{b\in\mathcal B}\sqrt{\alpha(b)}.
\]
The unanchored master update can then be written as
\[
\widetilde{\boldsymbol{\alpha}}_t
=
\argmin_{\boldsymbol{\alpha}\in\Delta_2}
\left\{
\left\langle
\boldsymbol{\alpha},
\widehat{\mathbf C}_{t-1}+\mathbf m_t
\right\rangle
+
\Phi_t(\boldsymbol{\alpha})
\right\},
\]
where
$\mathbf m_t=(m_t^b)_{b\in\mathcal B}$ is predictable before
the action is sampled. Introduce the corresponding look-ahead point
\[
\boldsymbol{\alpha}_t^+
=
\argmin_{\boldsymbol{\alpha}\in\Delta_2}
\left\{
\left\langle
\boldsymbol{\alpha},
\widehat{\mathbf C}_{t-1}
+
\widehat{\mathbf c}_t
\right\rangle
+
\Phi_t(\boldsymbol{\alpha})
\right\}.
\]
The residual between the observed scalarized cost and its predictable
part is
\[
\mathbf z_t
=
\widehat{\mathbf c}_t-\mathbf m_t.
\]

The optimistic-FTRL stability--penalty decomposition gives, for every
fixed $\boldsymbol{u}\in\Delta_2$,
\[
\begin{aligned}
	\sum_{t=1}^{T}
	\left\langle
	\widetilde{\boldsymbol{\alpha}}_t-\boldsymbol{u},
	\widehat{\mathbf c}_t
	\right\rangle
	\leq\;&
	\mathcal P_T(\boldsymbol{u})\\
	&+
	\sum_{t=1}^{T}
	\left[
	\left\langle
	\widetilde{\boldsymbol{\alpha}}_t
	-
	\boldsymbol{\alpha}_t^+,
	\mathbf z_t
	\right\rangle
	-
	D_{\Phi_t}
	\left(
	\boldsymbol{\alpha}_t^+,
	\widetilde{\boldsymbol{\alpha}}_t
	\right)
	\right],
\end{aligned}
\]
where $D_{\Phi_t}$ is the Bregman divergence generated by $\Phi_t$.
This inequality follows by applying the first-order optimality
conditions to
$\widetilde{\boldsymbol{\alpha}}_t$ and
$\boldsymbol{\alpha}_t^+$ and summing the resulting three-point
inequalities over $t$.

Because $\eta_t$ is nonincreasing, the changing-regularizer penalty
satisfies
\[
\mathcal P_T(\boldsymbol{u})
\leq
\frac{
	\max_{\boldsymbol{v}\in\Delta_2}\Psi(\boldsymbol{v})
	-
	\min_{\boldsymbol{v}\in\Delta_2}\Psi(\boldsymbol{v})
}{\eta_T}
=
\frac{D_\Psi}{\eta_T}.
\]

For an interior point
$\boldsymbol{\alpha}\in\Delta_2$,
\[
\nabla^2\Psi(\boldsymbol{\alpha})
=
\frac{1}{2}
\operatorname{diag}
\left(
\alpha(b)^{-3/2}
\right)_{b\in\mathcal B}.
\]
Hence the dual local norm induced by $\Psi$ is
\[
\|\mathbf z\|_{\ast,\boldsymbol{\alpha}}^2
=
\mathbf z^{\mathsf T}
[\nabla^2\Psi(\boldsymbol{\alpha})]^{-1}
\mathbf z
=
2\sum_{b\in\mathcal B}
[\alpha(b)]^{3/2}z(b)^2.
\]
Evaluated at
$\widetilde{\boldsymbol{\alpha}}_t$, this is precisely the local norm
used in Theorem~\ref{thm:pa_master}, and
\[
\|\mathbf z_t\|_{\ast,t}^2
=
2V_t.
\]

Under the assumed local-stability condition
$\eta_t\|\mathbf z_t\|_{\ast,t}\leq c_0<1$, the segment joining
$\widetilde{\boldsymbol{\alpha}}_t$ and
$\boldsymbol{\alpha}_t^+$ remains in a constant-factor local
neighborhood. Consequently, for some
$c_\Psi=c_\Psi(c_0)>0$,
\[
D_{\Phi_t}
\left(
\boldsymbol{\alpha}_t^+,
\widetilde{\boldsymbol{\alpha}}_t
\right)
\geq
\frac{c_\Psi}{2\eta_t}
\left\|
\boldsymbol{\alpha}_t^+
-
\widetilde{\boldsymbol{\alpha}}_t
\right\|_t^2.
\]
Local Hölder inequality then yields
\[
\begin{aligned}
	&
	\left\langle
	\widetilde{\boldsymbol{\alpha}}_t
	-
	\boldsymbol{\alpha}_t^+,
	\mathbf z_t
	\right\rangle
	-
	D_{\Phi_t}
	\left(
	\boldsymbol{\alpha}_t^+,
	\widetilde{\boldsymbol{\alpha}}_t
	\right)\\
	&\qquad\leq
	\left\|
	\widetilde{\boldsymbol{\alpha}}_t
	-
	\boldsymbol{\alpha}_t^+
	\right\|_t
	\|\mathbf z_t\|_{\ast,t}
	-
	\frac{c_\Psi}{2\eta_t}
	\left\|
	\widetilde{\boldsymbol{\alpha}}_t
	-
	\boldsymbol{\alpha}_t^+
	\right\|_t^2\\
	&\qquad\leq
	\frac{\eta_t}{2c_\Psi}
	\|\mathbf z_t\|_{\ast,t}^2
	=
	\frac{\eta_t}{c_\Psi}V_t.
\end{aligned}
\]
Thus, with $C_\Psi=1/c_\Psi$,
\[
\sum_{t=1}^{T}
\left\langle
\widetilde{\boldsymbol{\alpha}}_t-\boldsymbol{u},
\widehat{\mathbf c}_t
\right\rangle
\leq
\frac{D_\Psi}{\eta_T}
+
C_\Psi
\sum_{t=1}^{T}\eta_tV_t.
\]

It remains to replace the estimated costs by their conditional
expectations. From Lemma~\ref{lem:pa_unbiasedness},
\[
\mathbb E_t[\widehat{c}_t^b]
=
c_t^b-B_{t,\mathrm{fb}}^b.
\]
Since
$\widetilde{\boldsymbol{\alpha}}_t$ and
$\boldsymbol{u}$ are distributions over the two bases,
\[
\left\langle
\widetilde{\boldsymbol{\alpha}}_t-\boldsymbol{u},
\mathbf B_{t,\mathrm{fb}}
\right\rangle
\leq
\max_{b\in\mathcal B}B_{t,\mathrm{fb}}^b.
\]
Taking expectations and summing over $t$ therefore gives
\[
\begin{aligned}
	\mathbb E
	\left[
	\sum_{t=1}^{T}
	\left\langle
	\widetilde{\boldsymbol{\alpha}}_t-\boldsymbol{u},
	\mathbf c_t
	\right\rangle
	\right]
	\leq\;&
	\frac{D_\Psi}{\eta_T}
	+
	C_\Psi
	\sum_{t=1}^{T}
	\eta_t\mathbb E[V_t]\\
	&+
	\sum_{t=1}^{T}
	\max_{b\in\mathcal B}
	\mathbb E[B_{t,\mathrm{fb}}^b],
\end{aligned}
\]
which is the claim of Theorem~\ref{thm:pa_master}.

\section{Safe Projection and Oracle Decomposition}
\label{app:safe_oracle}

\subsection{Safe Projection}

We first prove Proposition~\ref{prop:pa_safe_projection}. For fixed
$t$, suppress the time index and write
$\overline{\boldsymbol{\pi}}$ for the strictly positive
pre-projection policy and $\mathbf h$ for the risk vector. The minimum
risk achievable over the simplex is
\[
\min_{\mathbf p\in\Delta_K}
\langle\mathbf p,\mathbf h\rangle
=
h_{\min}
\triangleq
\min_k h(k).
\]
Hence the feasible set is nonempty if and only if
$\tau\geq h_{\min}$.

Because
$D_{\mathrm{KL}}(\mathbf p\Vert\overline{\boldsymbol{\pi}})$
is strictly convex in $\mathbf p$ and the feasible set is convex and
compact, every feasible problem has a unique solution. If
$\langle\overline{\boldsymbol{\pi}},\mathbf h\rangle\leq\tau$,
the unprojected policy is feasible and attains zero KL divergence;
therefore it is the unique optimum.

Consider next
\[
h_{\min}
<
\tau
<
\langle\overline{\boldsymbol{\pi}},\mathbf h\rangle.
\]
The risk constraint must be active. The Lagrangian is
\[
\begin{aligned}
	\mathcal L(\mathbf p,\lambda,\nu)
	&=
	\sum_{k=1}^{K}
	p(k)\log\frac{p(k)}{\overline{\pi}(k)}
	\\
	&
	+
	\lambda
	\left(
	\sum_{k=1}^{K}p(k)h(k)-\tau
	\right)
	\\
	&
	+
	\nu
	\left(
	\sum_{k=1}^{K}p(k)-1
	\right),
\end{aligned}
\]
with $\lambda\geq0$. Stationarity gives
\[
\log\frac{p(k)}{\overline{\pi}(k)}
+
1+\lambda h(k)+\nu
=
0,
\]
and therefore
\[
p_\lambda(k)
=
\frac{
	\overline{\pi}(k)e^{-\lambda h(k)}
}{
	\sum_{j=1}^{K}
	\overline{\pi}(j)e^{-\lambda h(j)}
}.
\]
Define
\[
\varphi(\lambda)
=
\langle\mathbf p_\lambda,\mathbf h\rangle.
\]
Direct differentiation gives
\[
\varphi'(\lambda)
=
-
\operatorname{Var}_{\mathbf p_\lambda}[h]
\leq0.
\]
Unless all coordinates of $\mathbf h$ are equal,
$\varphi$ is strictly decreasing from
$\varphi(0)=
\langle\overline{\boldsymbol{\pi}},\mathbf h\rangle$
to $h_{\min}$ as $\lambda\rightarrow\infty$. Hence there is a unique
$\lambda>0$ satisfying $\varphi(\lambda)=\tau$, which proves the
exponential-tilting representation. At $\tau=h_{\min}$, the limiting
distribution is the renormalization of
$\overline{\boldsymbol{\pi}}$ over
$\argmin_k h(k)$.

For any vector $\boldsymbol{\kappa}$ with span at most $G_c$,
the total-variation inequality gives
\[
\left|
\left\langle
\mathbf p-\overline{\boldsymbol{\pi}},
\boldsymbol{\kappa}
\right\rangle
\right|
\leq
\frac{G_c}{2}
\|\mathbf p-\overline{\boldsymbol{\pi}}\|_1.
\]
Pinsker's inequality then yields
\[
\left|
\left\langle
\mathbf p-\overline{\boldsymbol{\pi}},
\boldsymbol{\kappa}
\right\rangle
\right|
\leq
G_c
\sqrt{
	\frac{1}{2}
	D_{\mathrm{KL}}
	(
	\mathbf p
	\Vert
	\overline{\boldsymbol{\pi}}
	)
}.
\]
Applying this inequality to the Safe optimizer proves the distortion
bound in Proposition~\ref{prop:pa_safe_projection}.

\subsection{Conditional Hit-Risk Guarantee}

We next prove Theorem~\ref{thm:pa_exact_risk}. Given the public
pre-action information $\mathcal P_t$, the sampled action and hidden
attack realization are conditionally independent. Therefore,
\[
\begin{aligned}
	\Pr(
	A_t\in\mathcal J_t
	\mid\mathcal P_t
	)
	&=
	\sum_{k=1}^{K}
	\Pr(A_t=k\mid\mathcal P_t)
	\Pr(k\in\mathcal J_t\mid\mathcal P_t)\\
	&=
	\sum_{k=1}^{K}
	p_t(k)h_t(k)
	=
	\langle\mathbf p_t,\mathbf h_t\rangle.
\end{aligned}
\]
On feasible Safe rounds, the final quantity is at most $\tau_t$ by
construction. If $\tau_t<h_{t,\min}$, then for every
$\mathbf p\in\Delta_K$,
\[
\langle\mathbf p,\mathbf h_t\rangle
\geq
h_{t,\min},
\]
so a violation of at least $h_{t,\min}-\tau_t$ is unavoidable. A
distribution supported on the minimum-risk channels attains this lower
bound.

If the projection uses an approximate vector
$\widetilde{\mathbf h}_t$ satisfying
$\|\widetilde{\mathbf h}_t-\mathbf h_t\|_\infty
\leq\varepsilon_t$, then
\[
\begin{aligned}
	\langle\mathbf p_t,\mathbf h_t\rangle
	&\leq
	\langle\mathbf p_t,\widetilde{\mathbf h}_t\rangle
	+
	\left|
	\left\langle
	\mathbf p_t,
	\mathbf h_t-\widetilde{\mathbf h}_t
	\right\rangle
	\right|\\
	&\leq
	\tau_t
	+
	\|\mathbf p_t\|_1
	\|\mathbf h_t-\widetilde{\mathbf h}_t\|_\infty\\
	&\leq
	\tau_t+\varepsilon_t,
\end{aligned}
\]
which proves the misspecification corollary.

\subsection{Oracle Decomposition}

We finally prove Theorem~\ref{thm:pa_oracle}. The scalarized cost of
the model-class mixture satisfies
\[
\begin{aligned}
	\left\langle
	\boldsymbol{\pi}_t^{\mathrm{mix}},
	\boldsymbol{\kappa}_t
	\right\rangle
	&=
	\left\langle
	\sum_{b\in\mathcal B}
	\alpha_t(b)\boldsymbol{\rho}_t^b,
	\boldsymbol{\kappa}_t
	\right\rangle\\
	&=
	\sum_{b\in\mathcal B}
	\alpha_t(b)c_t^b
	=
	\langle\boldsymbol{\alpha}_t,\mathbf c_t\rangle.
\end{aligned}
\]

For any comparator
$\boldsymbol{u}\in\Delta_2$, decompose the executed-policy cost as
\[
\begin{aligned}
	&
	\left\langle
	\mathbf p_t,
	\boldsymbol{\kappa}_t
	\right\rangle
	-
	\langle\boldsymbol{u},\mathbf c_t\rangle\\
	&=
	\left\langle
	\boldsymbol{\alpha}_t-\boldsymbol{u},
	\mathbf c_t
	\right\rangle\\
	&\quad+
	\left\langle
	\overline{\boldsymbol{\pi}}_t
	-
	\boldsymbol{\pi}_t^{\mathrm{mix}},
	\boldsymbol{\kappa}_t
	\right\rangle\\
	&\quad+
	\left\langle
	\mathbf p_t
	-
	\overline{\boldsymbol{\pi}}_t,
	\boldsymbol{\kappa}_t
	\right\rangle.
\end{aligned}
\]

The first term is bounded by Theorem~\ref{thm:pa_master}, together with
the anchor and Local-prior perturbations:
\[
\begin{aligned}
	\mathbb E
	\left[
	\sum_{t=1}^{T}
	\left\langle
	\boldsymbol{\alpha}_t-\boldsymbol{u},
	\mathbf c_t
	\right\rangle
	\right]
	\leq\;&
	\frac{D_\Psi}{\eta_T}
	+
	C_\Psi
	\sum_{t=1}^{T}
	\eta_t\mathbb E[V_t]\\
	&+
	E_T^{\mathrm{anchor}}
	+
	E_T^{\mathrm{prior}}
	+
	E_T^{\mathrm{fb}}.
\end{aligned}
\]

Since
\[
\overline{\boldsymbol{\pi}}_t
=
(1-\gamma_t)
\boldsymbol{\pi}_t^{\mathrm{mix}}
+
\gamma_t\mathbf e_t,
\]
and $\boldsymbol{\kappa}_t$ has span at most $G_c$,
\[
\left|
\left\langle
\overline{\boldsymbol{\pi}}_t
-
\boldsymbol{\pi}_t^{\mathrm{mix}},
\boldsymbol{\kappa}_t
\right\rangle
\right|
\leq
G_c\gamma_t.
\]
Summing gives $E_T^{\mathrm{explore}}$.

On feasible Safe rounds, Proposition~\ref{prop:pa_safe_projection}
gives
\[
\left|
\left\langle
\mathbf p_t
-
\overline{\boldsymbol{\pi}}_t,
\boldsymbol{\kappa}_t
\right\rangle
\right|
\leq
G_c
\sqrt{
	\frac{1}{2}
	D_{\mathrm{KL}}
	(
	\mathbf p_t
	\Vert
	\overline{\boldsymbol{\pi}}_t
	)
}.
\]
Its sum is $E_T^{\mathrm{safe}}$. Support loss, infeasible budgets, and
fallback execution are collected in
$E_T^{\mathrm{fallback}}$.

Choosing $\boldsymbol{u}$ as the unit vector corresponding to the
better shadow base therefore yields
\[
\begin{aligned}
	&
	\mathbb E
	\left[
	\sum_{t=1}^{T}
	\left\langle
	\mathbf p_t,
	\boldsymbol{\kappa}_t
	\right\rangle
	\right]
	-
	\min_{b\in\mathcal B}C_T^b\\
	&\quad\leq
	\frac{D_\Psi}{\eta_T}
	+
	C_\Psi
	\sum_{t=1}^{T}
	\eta_t\mathbb E[V_t]
	+
	E_T^{\mathrm{anchor}}
	+
	E_T^{\mathrm{prior}}\\
	&\qquad\quad+
	E_T^{\mathrm{explore}}
	+
	E_T^{\mathrm{safe}}
	+
	E_T^{\mathrm{fallback}},
\end{aligned}
\]
which proves the common-trajectory oracle decomposition.

For independently executed references, define
\[
E_T^{\mathrm{transfer}}
=
\max_{b\in\mathcal B}
[C_T^b-C_T^{b,\circ}]_+.
\]
For every $b$,
$C_T^b\leq C_T^{b,\circ}+E_T^{\mathrm{transfer}}$, and hence
\[
\min_b C_T^b
\leq
\min_b C_T^{b,\circ}
+
E_T^{\mathrm{transfer}}.
\]
Adding this inequality to the common-trajectory bound establishes the
standalone-base extension.

Finally, under the assumptions of
Corollary~\ref{cor:pa_sublinear},
\[
\frac{D_\Psi}{\eta_T}
=
O(\sqrt{T}),
\qquad
\sum_{t=1}^{T}\eta_t\mathbb E[V_t]
=
O(\sqrt{T}),
\]
and both the anchor and exploration sums are $O(\sqrt{T})$ when
$a_{\min}=0$. If the remaining Safe, fallback, and transfer terms are
$o(T)$, division by $T$ gives vanishing average excess cost. If all
three are $\widetilde O(\sqrt T)$, the total excess cost is
$\widetilde O(\sqrt T)$.


\section{Supplementary Experimental Results}
\label{app:additional_experiments}

This appendix reports the saved full-run aggregates underlying the
robustness, scaling, and Safe-calibration statements in
Section~\ref{sec:experiments}. No simulation or parameter search was rerun.
Unless stated otherwise, results use 20 evaluation seeds, matched across
methods within each comparison; 95\% confidence intervals use the convention
specified in Section~\ref{sec:experiments}. The scaling sweep uses 10 seeds
per configuration, and the Safe calibration uses 20 calibration seeds that
are disjoint from the main evaluation seeds.

\subsection{Cross-Attacker Robustness}
\label{app:cross_attacker}

The full method-by-attacker comparison covers random, sweeping,
FOLPETTI-inspired, and contextual white-box jammers. Table~\ref{tab:app_cross_attacker}
reports every saved method in the common order used by the heatmaps and
tradeoff plots. Under random jamming, the three D-PACT modes coincide at
$3.580$ Mbps and empirical hit rate $0.0834$, consistent with the
action-independent baseline $M/K=1/12$. The white-box column separates the
communication/exposure tradeoff: D-PACT-Hit records $3.230$ Mbps at empirical
hit rate $0.1352$, whereas D-PACT-Safe95 records $3.105$ Mbps at $0.1080$.

\begin{table*}[!t]
\centering
\caption{Complete cross-attacker comparison from the saved full evaluation.
Each entry is mean goodput in Mbps / empirical jam-hit rate over 20 evaluation
seeds. Methods in each attacker column share the same seed list and
realizations.}
\label{tab:app_cross_attacker}
\small
\setlength{\tabcolsep}{7pt}
\begin{tabular}{@{}lcccc@{}}
\toprule
Method & Random & Sweep & FOLPETTI-inspired & White-box \\
\midrule
Thompson & 2.481/0.0837 & 2.479/0.0838 & 2.468/0.0884 & 2.399/0.1049 \\
EXP3 & 2.487/0.0833 & 2.476/0.0849 & 2.429/0.1043 & 2.336/0.1306 \\
LC-INF & 3.592/0.0836 & 3.593/0.0845 & 3.596/0.0836 & 3.231/0.1723 \\
RA-EXP4 & 3.067/0.0828 & 3.063/0.0836 & 3.061/0.0843 & 2.934/0.1048 \\
AUFH & 2.482/0.0829 & 2.478/0.0829 & 2.471/0.0844 & 2.441/0.0894 \\
D-PACT-Base & 3.580/0.0834 & 3.574/0.0848 & 3.576/0.0842 & 3.229/0.1607 \\
D-PACT-Hit & 3.580/0.0834 & 3.568/0.0842 & 3.573/0.0842 & 3.230/0.1352 \\
D-PACT-Safe95 & 3.580/0.0834 & 3.697/0.0481 & 3.581/0.0793 & 3.105/0.1080 \\
\bottomrule
\end{tabular}
\end{table*}

\begin{figure*}[!t]
  \centering
  \includegraphics[width=0.96\textwidth]
  {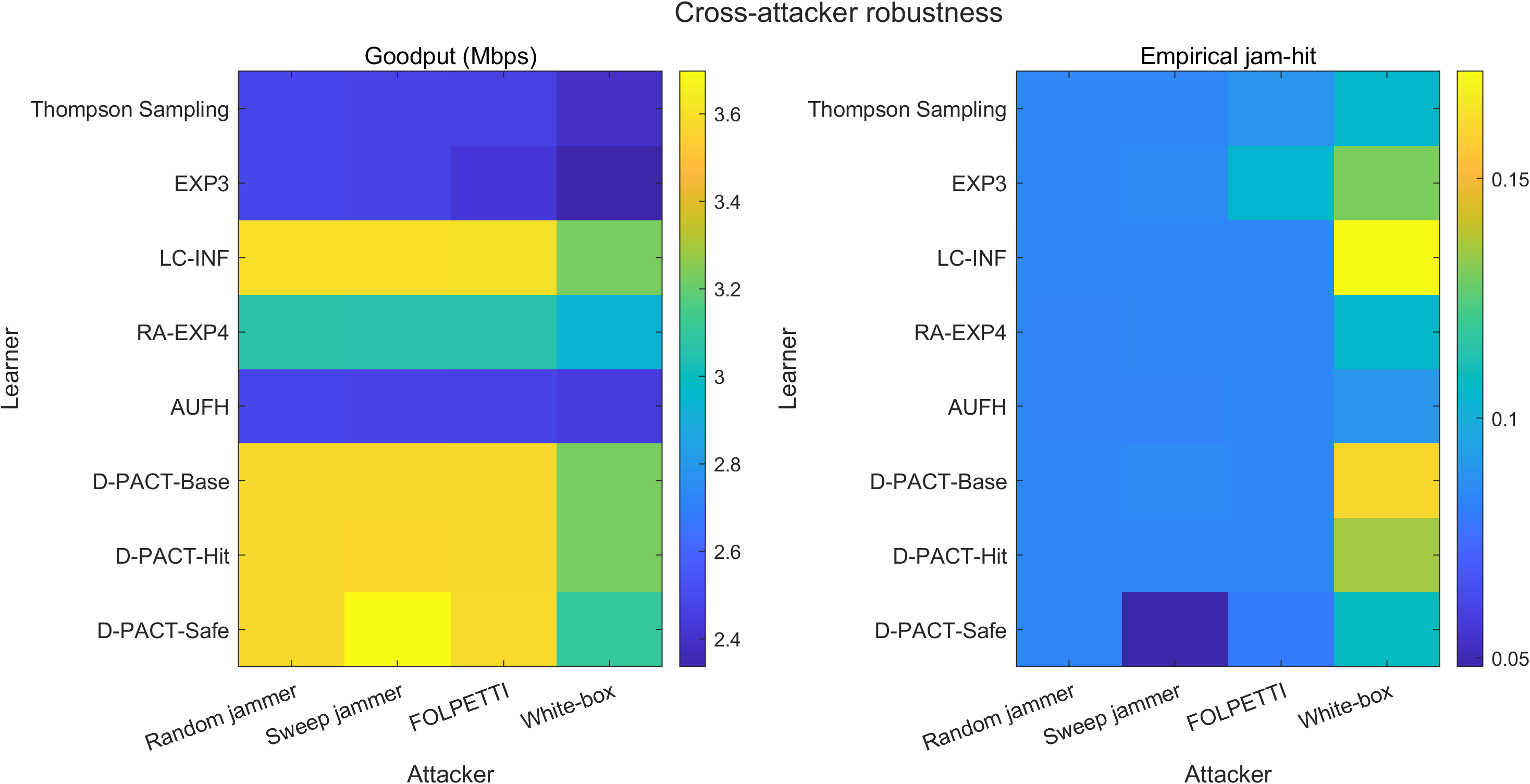}
  \caption{Saved full-run cross-attacker heatmaps. The left panel reports
  mean goodput in Mbps and the right panel reports the empirical jam-hit
  frequency. Rows contain all eight evaluated methods and columns contain
  the four jammer types listed in Table~\ref{tab:app_cross_attacker}.
  LC-INF, RA-EXP4, and AUFH abbreviate LC-Tsallis-INF-Online, risk-aware EXP4,
  and AUFH-EXP3++, respectively.}
  \label{fig:app_cross_attacker_heatmaps}
\end{figure*}

\begin{figure*}[!t]
  \centering
  \includegraphics[width=0.90\textwidth]
  {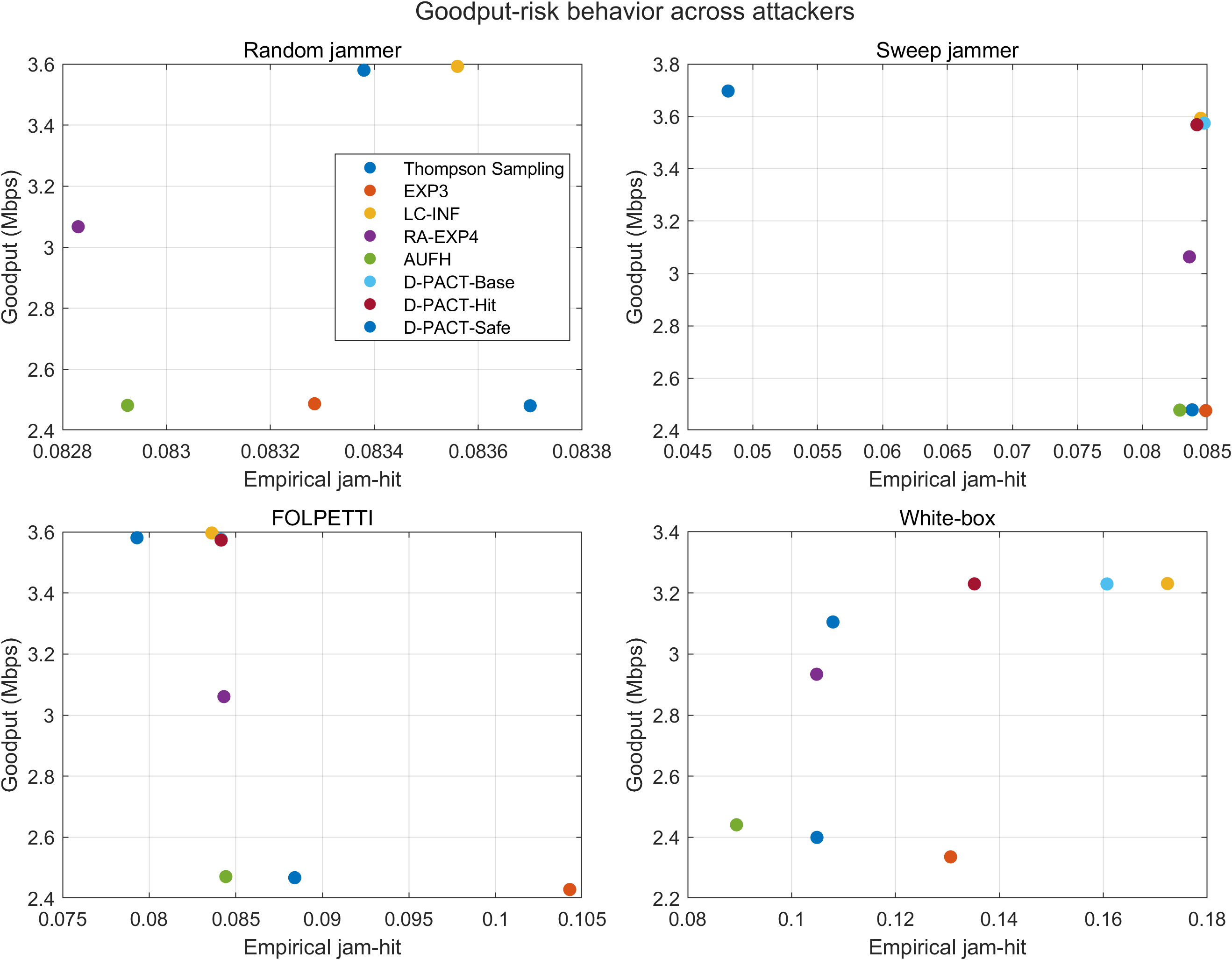}
  \caption{Saved full-run goodput--exposure operating points for the random,
  sweep, FOLPETTI-inspired, and contextual white-box jammers. Each marker is
  the 20-seed mean reported in Table~\ref{tab:app_cross_attacker}; the axes use
  empirical jam-hit frequency rather than conditional expected risk.}
  \label{fig:app_cross_attacker_tradeoff}
\end{figure*}

\FloatBarrier

\raggedbottom
\subsection{Regime Robustness}
\label{app:regime_robustness}

The saved regime suite separates communication-model mismatch without active
jamming from attacked environmental regimes. In the nominal, nonlinear, and
observable-switching cases, Table~\ref{tab:app_model_regimes} reports goodput
and, where defined, the Local-base master mass. The high Local mass in
observable switching and the smaller mass in the nonlinear negative control
give the detailed values behind the model-adaptation comparison in
Fig.~\ref{fig:model_selection}. The identical Base, Hit, and Safe entries in
these no-jammer rows reflect zero attack risk; they are retained rather than
discarded.

\begin{table}[H]
\centering
\caption{Saved no-jammer model-regime results. Each entry is mean goodput in
Mbps / mean Local-base master mass over 20 evaluation seeds; a dash denotes a
method without the D-PACT master.}
\label{tab:app_model_regimes}
\scriptsize
\setlength{\tabcolsep}{1.8pt}
\begin{tabular}{@{}lccc@{}}
\toprule
Method & Nominal & Nonlinear interaction & Observable switching \\
\midrule
LC-INF & 3.915/-- & 3.495/-- & 2.450/-- \\
RA-EXP4 & 3.360/-- & 2.527/-- & 2.125/-- \\
AUFH & 2.709/-- & 0.945/-- & 1.296/-- \\
D-PACT-Base & 3.898/0.157 & 3.476/0.258 & 2.890/0.964 \\
D-PACT-Hit & 3.898/0.157 & 3.476/0.258 & 2.890/0.964 \\
D-PACT-Safe95 & 3.898/0.157 & 3.476/0.258 & 2.890/0.964 \\
\bottomrule
\end{tabular}
\end{table}

For the attacked regimes, Tables~\ref{tab:app_attacked_regimes_a} and
\ref{tab:app_attacked_regimes_b} retain every
saved method, including the lower-goodput outcomes. The reported hit quantity
is the realized empirical jam-hit frequency. It is distinct from the
conditional expected hit risk used by the Safe projection and from a
numerical policy-budget violation. In particular, D-PACT-Safe95 remains near
the $0.11$ conditional budget across the contaminated, exogenous-sweep,
mixed, and hidden-Markov settings, while realized empirical frequencies vary
with the sampled trajectories.

\begin{table}[H]
\centering
\caption{Saved attacked-regime comparison (part 1). Each entry is mean
goodput in Mbps / empirical jam-hit rate over 20 evaluation seeds.}
\label{tab:app_attacked_regimes_a}
\scriptsize
\setlength{\tabcolsep}{2.1pt}
\begin{tabular}{@{}lccc@{}}
\toprule
Method & White-box & Contaminated & Exog. sweep \\
\midrule
LC-INF & 3.235/0.1729 & 3.192/0.1806 & 3.160/0.1882 \\
RA-EXP4 & 2.938/0.1054 & 2.873/0.1102 & 2.835/0.1147 \\
AUFH & 2.441/0.0903 & 2.327/0.0908 & 2.245/0.0906 \\
D-PACT-Base & 3.233/0.1588 & 3.191/0.1641 & 3.161/0.1704 \\
D-PACT-Hit & 3.235/0.1337 & 3.195/0.1408 & 3.166/0.1462 \\
D-PACT-Safe95 & 3.104/0.1081 & 3.019/0.1088 & 2.967/0.1093 \\
\bottomrule
\end{tabular}
\end{table}

\begin{table}[H]
\centering
\caption{Saved attacked-regime comparison (part 2), using the same metric,
seed count, and method order as Table~\ref{tab:app_attacked_regimes_a}.}
\label{tab:app_attacked_regimes_b}
\scriptsize
\setlength{\tabcolsep}{4pt}
\begin{tabular}{@{}lcc@{}}
\toprule
Method & Mixed & Hidden Markov \\
\midrule
LC-INF & 2.947/0.2361 & 2.679/0.1749 \\
RA-EXP4 & 2.582/0.1378 & 2.292/0.1040 \\
AUFH & 2.289/0.1126 & 1.911/0.1297 \\
D-PACT-Base & 2.975/0.1951 & 2.668/0.1697 \\
D-PACT-Hit & 2.983/0.1759 & 2.568/0.1379 \\
D-PACT-Safe95 & 2.571/0.1093 & 2.398/0.1093 \\
\bottomrule
\end{tabular}
\end{table}

\begin{figure*}[!t]
  \centering
  \includegraphics[width=0.94\textwidth]
  {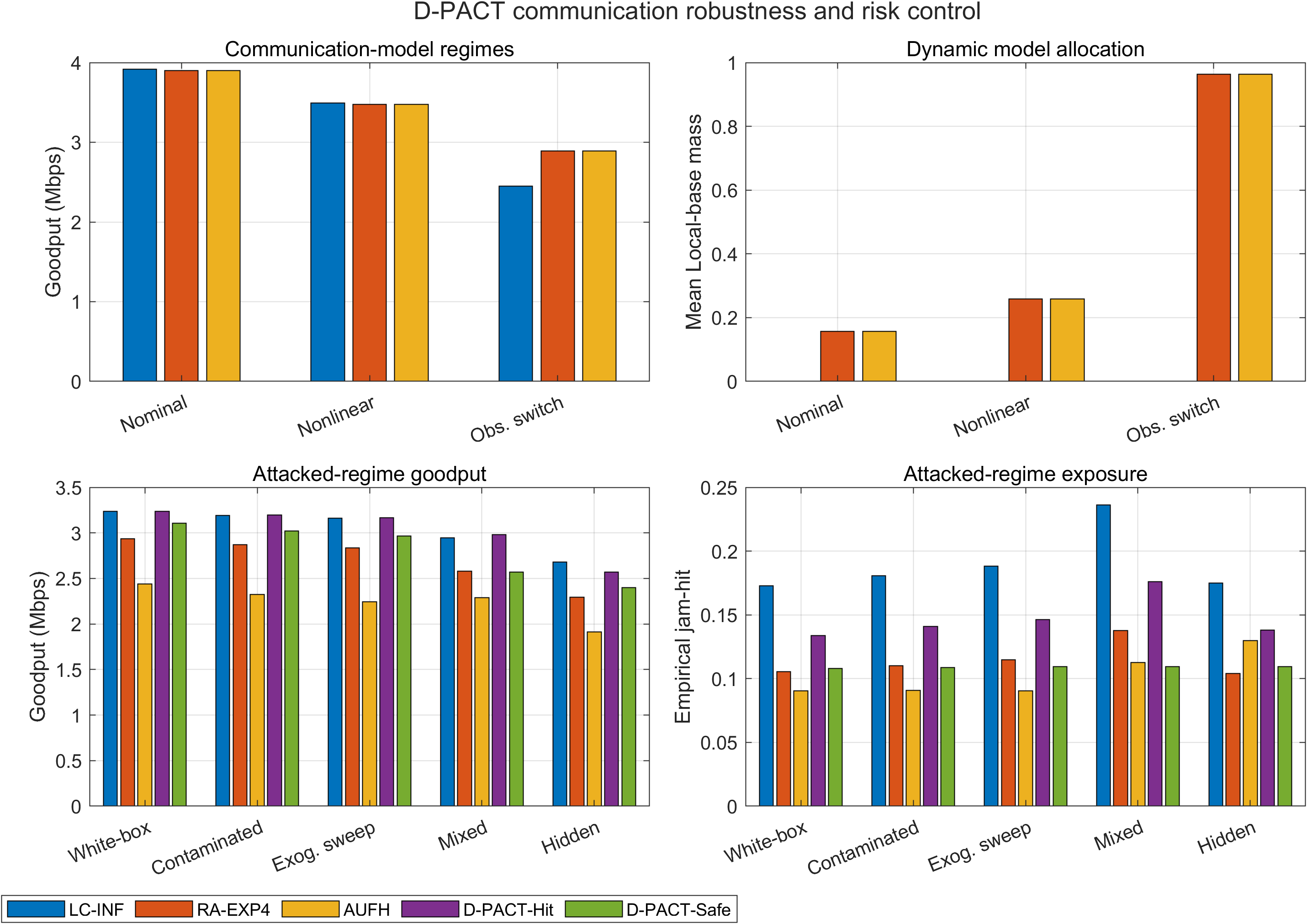}
  \caption{Saved full-run regime summary. The upper panels show goodput and
  Local-base allocation for nominal, nonlinear, and observable-switching
  communication models without active jamming. The lower panels show mean
  goodput and empirical jam-hit frequency under the contextual white-box
  jammer in the nominal white-box, contaminated, exogenous-sweep, mixed, and
  hidden-Markov environmental regimes. Exact complete comparisons are given
  in Tables~\ref{tab:app_model_regimes},
  \ref{tab:app_attacked_regimes_a}, and
  \ref{tab:app_attacked_regimes_b}.}
  \label{fig:app_regime_summary}
\end{figure*}

\subsection{Scaling and Runtime}
\label{app:scaling_runtime}

The saved scaling sweep varies $K\in\{8,12,16,32\}$ at $T=10^4$ and
$T\in\{10^3,5\!\times\!10^3,10^4,2\!\times\!10^4\}$ at $K=12$.
Fig.~\ref{fig:app_scaling} retains the comparative curves, while
Table~\ref{tab:app_scaling} gives the D-PACT-Hit and D-PACT-Safe95 operating
points with conditional expected risk, empirical jam-hit frequency, and
runtime. The Safe budget is scaled with $K$ in the channel-count sweep and is
$0.11$ in the horizon sweep. At $K=32$, the unprojected Hit policy is already
below the scaled Safe budget, so the two goodputs are nearly identical.

Runtime is measured in seconds on the hardware stated in
Section~\ref{sec:experiments}: MATLAB R2020b on a 12th Gen Intel Core
i7-12650H processor with 16 GB DDR4 memory and 16 MATLAB workers. The saved
timing excludes environment generation, aggregation, plotting, and file
output, and is therefore a relative algorithmic-overhead measurement rather
than an end-to-end wall-clock benchmark.

\begin{table*}[!t]
\centering
\caption{Detailed saved D-PACT scaling results (10 seeds per configuration).
Expected hit is the conditional marginal policy risk; empirical hit is the
realized jam-hit frequency. Runtime is in seconds and uses the timing scope
stated in Section~\ref{sec:experiments}.}
\label{tab:app_scaling}
\small
\setlength{\tabcolsep}{5pt}
\begin{tabular}{@{}cclccccc@{}}
\toprule
Sweep & Value & Method & $\tau$ & Goodput (Mbps) & Expected hit & Empirical hit & Runtime (s) \\
\midrule
K & 8 & D-PACT-Hit & -- & 2.931 & 0.1978 & 0.1994 & 84.9 \\
K & 8 & D-PACT-Safe95 & 0.1505 & 2.758 & 0.1489 & 0.1513 & 87.4 \\
\addlinespace[1pt]
K & 12 & D-PACT-Hit & -- & 3.235 & 0.1336 & 0.1348 & 78.7 \\
K & 12 & D-PACT-Safe95 & 0.1100 & 3.111 & 0.1081 & 0.1074 & 101.1 \\
\addlinespace[1pt]
K & 16 & D-PACT-Hit & -- & 3.397 & 0.0997 & 0.0990 & 110.7 \\
K & 16 & D-PACT-Safe95 & 0.0898 & 3.316 & 0.0868 & 0.0853 & 113.8 \\
\addlinespace[1pt]
K & 32 & D-PACT-Hit & -- & 3.639 & 0.0459 & 0.0467 & 114.4 \\
K & 32 & D-PACT-Safe95 & 0.0594 & 3.638 & 0.0458 & 0.0466 & 73.2 \\
\addlinespace[1pt]
T & 1000 & D-PACT-Hit & -- & 3.306 & 0.1122 & 0.1044 & 2.9 \\
T & 1000 & D-PACT-Safe95 & 0.1100 & 3.232 & 0.1032 & 0.0977 & 3.0 \\
\addlinespace[1pt]
T & 5000 & D-PACT-Hit & -- & 3.261 & 0.1301 & 0.1286 & 15.9 \\
T & 5000 & D-PACT-Safe95 & 0.1100 & 3.141 & 0.1074 & 0.1071 & 16.4 \\
\addlinespace[1pt]
T & 10000 & D-PACT-Hit & -- & 3.235 & 0.1336 & 0.1348 & 78.7 \\
T & 10000 & D-PACT-Safe95 & 0.1100 & 3.111 & 0.1081 & 0.1074 & 101.1 \\
\addlinespace[1pt]
T & 20000 & D-PACT-Hit & -- & 3.223 & 0.1359 & 0.1366 & 122.5 \\
T & 20000 & D-PACT-Safe95 & 0.1100 & 3.085 & 0.1085 & 0.1093 & 207.0 \\
\bottomrule
\end{tabular}
\end{table*}

\begin{figure*}[!t]
  \centering
  \includegraphics[width=0.92\textwidth]
  {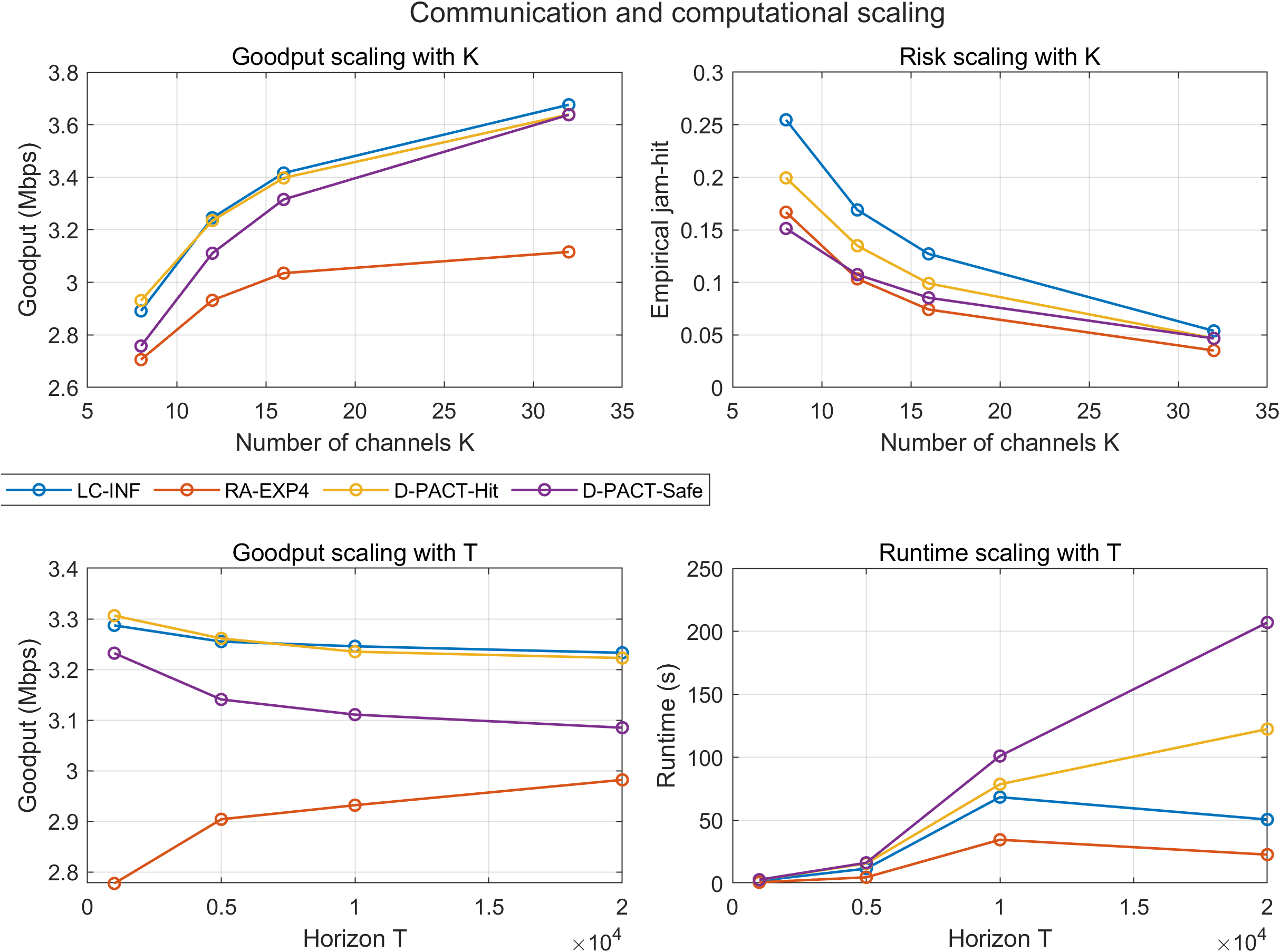}
  \caption{Saved 10-seed scaling results. The upper panels show goodput and
  empirical jam-hit frequency versus the number of channels. The lower panels
  show goodput and recorded runtime versus the horizon. Runtime uses the scope
  described in the text and Table~\ref{tab:app_scaling}.}
  \label{fig:app_scaling}
\end{figure*}

\subsection{Additional Ablation and Safe Calibration}
\label{app:safe_calibration}

The mechanism ablation is already reported in
Table~\ref{tab:ablation}, and the model-selection effect is already shown in
Fig.~\ref{fig:model_selection}; duplicating those displays would not add an
independent result. Table~\ref{tab:app_safe_calibration} instead supplies the
saved calibration diagnostics omitted from the main presentation. Here,
expected hit is the conditional marginal policy risk, pre-risk and post-risk
are the corresponding values before and after KL projection, and maximum
violation is the numerical excess of post-projection risk over the budget.
These quantities must not be identified with the empirical jam-hit frequency.

The calibrated Safe95 point at $\tau=0.11$ has $3.111$ Mbps in the independent
calibration scan, while the disjoint main evaluation reports approximately
$3.104$ Mbps in Table~\ref{tab:ablation}. The two saved values are deliberately
kept separate because they arise from different seed sets. Across the scan,
the recorded maximum policy-budget violation is at most $10^{-10}$.

\begin{table*}[!t]
\centering
\caption{Saved independent 20-seed Safe calibration scan under the contextual
white-box jammer. Goodput is mean $\pm$ 95\% confidence-interval half-width.
The budget applies to post-projection conditional risk, not to the realized
empirical hit frequency.}
\label{tab:app_safe_calibration}
\scriptsize
\setlength{\tabcolsep}{3.5pt}
\begin{tabular}{@{}ccccccccc@{}}
\toprule
$\tau$ & Goodput (Mbps) & Expected hit & Empirical hit & Pre-risk & Post-risk & Activation & Max violation & Runtime (s) \\
\midrule
0.085 & $2.662\!\pm\!0.006$ & 0.0850 & 0.0846 & 0.1088 & 0.0850 & 0.996 & 1.00e-10 & 27.1 \\
0.090 & $2.811\!\pm\!0.005$ & 0.0899 & 0.0889 & 0.1140 & 0.0899 & 0.981 & 1.00e-10 & 27.6 \\
0.100 & $3.006\!\pm\!0.005$ & 0.0994 & 0.0990 & 0.1223 & 0.0994 & 0.898 & 1.00e-10 & 27.5 \\
0.110 & $3.111\!\pm\!0.006$ & 0.1081 & 0.1073 & 0.1277 & 0.1081 & 0.768 & 1.00e-10 & 28.0 \\
0.120 & $3.168\!\pm\!0.005$ & 0.1154 & 0.1143 & 0.1306 & 0.1154 & 0.610 & 1.00e-10 & 27.3 \\
0.140 & $3.214\!\pm\!0.008$ & 0.1252 & 0.1243 & 0.1329 & 0.1252 & 0.318 & 9.99e-11 & 27.0 \\
\bottomrule
\end{tabular}
\end{table*}

\FloatBarrier

\bibliographystyle{IEEEtran}
\bibliography{references}

@article{wang2012ufh,
  author  = {Qian Wang and Ping Xu and Kui Ren and Xiang-Yang Li},
  title   = {Towards Optimal Adaptive {UFH}-Based Anti-Jamming Wireless Communication},
  journal = {IEEE Journal on Selected Areas in Communications},
  volume  = {30},
  number  = {1},
  pages   = {16--30},
  month   = jan,
  year    = {2012},
  doi     = {10.1109/JSAC.2012.120103}
}

@article{zhou2016unknown,
  author  = {Pan Zhou and Tao Jiang},
  title   = {Toward Optimal Adaptive Wireless Communications in Unknown Environments},
  journal = {IEEE Transactions on Wireless Communications},
  volume  = {15},
  number  = {5},
  pages   = {3655--3667},
  month   = may,
  year    = {2016},
  doi     = {10.1109/TWC.2016.2524638}
}

@article{hanawal2016joint,
  author  = {Manjesh K. Hanawal and Mohammad J. Abdel-Rahman and Marwan Krunz},
  title   = {Joint Adaptation of Frequency Hopping and Transmission Rate for Anti-Jamming Wireless Systems},
  journal = {IEEE Transactions on Mobile Computing},
  volume  = {15},
  number  = {9},
  pages   = {2247--2259},
  month   = sep,
  year    = {2016},
  doi     = {10.1109/TMC.2015.2492556}
}

@inproceedings{odeyomi2020mitigating,
  author    = {Olusola T. Odeyomi},
  title     = {Mitigating Jamming Attacks in Uncoordinated Frequency Hopping Using Multi-Armed Bandit},
  booktitle = {2020 11th IEEE Annual Ubiquitous Computing, Electronics \& Mobile Communication Conference (UEMCON)},
  pages     = {207--212},
  year      = {2020},
  doi       = {10.1109/UEMCON51285.2020.9298038}
}

@article{xu2020intelligent,
  author  = {Jianliang Xu and Huaxun Lou and Weifeng Zhang and Gaoli Sang},
  title   = {An Intelligent Anti-Jamming Scheme for Cognitive Radio Based on Deep Reinforcement Learning},
  journal = {IEEE Access},
  volume  = {8},
  pages   = {202563--202572},
  year    = {2020},
  doi     = {10.1109/ACCESS.2020.3036027}
}

@article{qi2024hopping,
  author  = {Jie Qi and Hongming Zhang and Xiaolei Qi and Mugen Peng},
  title   = {Deep Reinforcement Learning Based Hopping Strategy for Wideband Anti-Jamming Wireless Communications},
  journal = {IEEE Transactions on Vehicular Technology},
  volume  = {73},
  number  = {3},
  pages   = {2131--2144},
  month   = mar,
  year    = {2024},
  doi     = {10.1109/TVT.2023.3324387}
}

@inproceedings{kato2025lc,
  author    = {Masahiro Kato and Shinji Ito},
  title     = {{LC-Tsallis-INF}: Generalized Best-of-Both-Worlds Linear Contextual Bandits},
  booktitle = {Proceedings of the 28th International Conference on Artificial Intelligence and Statistics},
  series    = {Proceedings of Machine Learning Research},
  volume    = {258},
  pages     = {3655--3663},
  publisher = {PMLR},
  year      = {2025}
}

@inproceedings{takemura2021misspecified,
  author    = {Kei Takemura and Shinji Ito and Daisuke Hatano and Hanna Sumita and Takuro Fukunaga and Naonori Kakimura and Ken-ichi Kawarabayashi},
  title     = {A Parameter-Free Algorithm for Misspecified Linear Contextual Bandits},
  booktitle = {Proceedings of the 24th International Conference on Artificial Intelligence and Statistics},
  series    = {Proceedings of Machine Learning Research},
  volume    = {130},
  pages     = {3367--3375},
  publisher = {PMLR},
  year      = {2021}
}

@inproceedings{foster2019model,
  author    = {Dylan J. Foster and Akshay Krishnamurthy and Haipeng Luo},
  title     = {Model Selection for Contextual Bandits},
  booktitle = {Advances in Neural Information Processing Systems},
  volume    = {32},
  pages     = {14714--14725},
  year      = {2019}
}

@inproceedings{muthukumar2022model,
  author    = {Vidya K. Muthukumar and Akshay Krishnamurthy},
  title     = {Universal and Data-Adaptive Algorithms for Model Selection in Linear Contextual Bandits},
  booktitle = {Proceedings of the 39th International Conference on Machine Learning},
  series    = {Proceedings of Machine Learning Research},
  volume    = {162},
  pages     = {16197--16222},
  publisher = {PMLR},
  year      = {2022}
}

@inproceedings{agarwal2017corral,
  author    = {Alekh Agarwal and Haipeng Luo and Behnam Neyshabur and Robert E. Schapire},
  title     = {Corralling a Band of Bandit Algorithms},
  booktitle = {Proceedings of the 30th Conference on Learning Theory},
  series    = {Proceedings of Machine Learning Research},
  volume    = {65},
  pages     = {12--38},
  publisher = {PMLR},
  year      = {2017}
}

@inproceedings{pacchiano2020model,
  author    = {Aldo Pacchiano and My Phan and Yasin Abbasi-Yadkori and Anup Rao and Julian Zimmert and Tor Lattimore and Csaba Szepesv{\'a}ri},
  title     = {Model Selection in Contextual Stochastic Bandit Problems},
  booktitle = {Advances in Neural Information Processing Systems},
  volume    = {33},
  pages     = {10328--10337},
  year      = {2020}
}

@article{amuru2016jamming,
  author  = {SaiDhiraj Amuru and Cem Tekin and Mihaela van der Schaar and R. Michael Buehrer},
  title   = {Jamming Bandits---A Novel Learning Method for Optimal Jamming},
  journal = {IEEE Transactions on Wireless Communications},
  volume  = {15},
  number  = {4},
  pages   = {2792--2808},
  month   = apr,
  year    = {2016},
  doi     = {10.1109/TWC.2015.2510643}
}

@inproceedings{bout2022folpetti,
  author    = {Emilie Bout and Alessandro Brighente and Mauro Conti and Valeria Loscri},
  title     = {{FOLPETTI}: A Novel Multi-Armed Bandit Smart Attack for Wireless Networks},
  booktitle = {Proceedings of the 17th International Conference on Availability, Reliability and Security},
  pages     = {1--10},
  note      = {Art. no. 21},
  year      = {2022},
  doi       = {10.1145/3538969.3539001}
}

@inproceedings{sun2017safety,
  author    = {Wen Sun and Debadeepta Dey and Ashish Kapoor},
  title     = {Safety-Aware Algorithms for Adversarial Contextual Bandit},
  booktitle = {Proceedings of the 34th International Conference on Machine Learning},
  series    = {Proceedings of Machine Learning Research},
  volume    = {70},
  pages     = {3280--3288},
  publisher = {PMLR},
  year      = {2017}
}

@article{thornton2022constrained,
  author  = {Charles E. Thornton and R. Michael Buehrer and Anthony F. Martone},
  title   = {Constrained Contextual Bandit Learning for Adaptive Radar Waveform Selection},
  journal = {IEEE Transactions on Aerospace and Electronic Systems},
  volume  = {58},
  number  = {2},
  pages   = {1133--1148},
  month   = apr,
  year    = {2022},
  doi     = {10.1109/TAES.2021.3109110}
}

\vfill
\end{document}